\documentclass[]{sn-jnl}
\usepackage{graphicx}%
\usepackage{multirow}%
\usepackage{amsmath,amssymb,amsfonts}%
\usepackage{amsthm}%
\numberwithin{equation}{section}
\usepackage{tabularx}
\usepackage{booktabs} 
\usepackage{mathrsfs}%
\usepackage[title]{appendix}%
\usepackage{xcolor}%
\usepackage{textcomp}%
\usepackage{manyfoot}%
\usepackage{booktabs}%
\usepackage{algorithm}%
\usepackage{algorithmicx}%
\usepackage{algpseudocode}%
\usepackage{upgreek}
\usepackage{listings}%
\usepackage{lineno} 
\usepackage{afterpage}
\usepackage{comment}
\usepackage{geometry}
\usepackage[table]{xcolor}
\def\ve{\varepsilon}
\def\tilde{\widetilde}
\def\emp{\emptyset}

\def\ox{\overline{x}}

\def\Hat{\widehat}

\def\Bar{\overline}

\def\ve{\varepsilon}
\def\epsilon{\varepsilon}
\def\ox{\bar{x}}

\def\dn{\downarrow}
\def\O{\Omega}

\def\emp{\emptyset}

\def\lm{\lambda}

\usepackage{tikz}
\usepackage{subfigure}
\usepackage{caption}

\usepackage{adjustbox}
\usetikzlibrary{shapes.geometric, arrows}
\tikzstyle{startstop} = [rectangle, rounded corners, 
minimum width=3cm, 
minimum height=1cm,
text centered, 
draw=black, 
fill=red!30]

\tikzstyle{io} = [trapezium, 
trapezium stretches=true, 
trapezium left angle=70, 
trapezium right angle=110, 
minimum width=3cm, 
minimum height=1cm, text centered, 
draw=black, fill=blue!30]

\tikzstyle{process} = [rectangle, 
minimum width=3cm, 
minimum height=1cm, 
text centered, 
text width=3cm, 
draw=black, 
fill=orange!30]

\tikzstyle{decision} = [diamond, 
minimum width=3cm, 
minimum height=2cm,
text centered, 
draw=black, 
fill=green!30]
\tikzstyle{arrow} = [thick,->,>=stealth]
\usepackage[style=apa, backend=biber]{biblatex}
\theoremstyle{thmstyleone}%
\newtheorem{theorem}{Theorem}[section]
\newtheorem{lemma}[theorem]{Lemma}

\newtheorem{proposition}[theorem]{Proposition}

\theoremstyle{thmstyletwo}%

\theoremstyle{thmstylethree}%

\begin{document}

\title[Article Title]{New Location Science Models with Applications to UAV-Based Disaster Relief}

\author[a]{\fnm{Sina} \sur{Kazemdehbashi}}

\author[a]{\fnm{Yanchao} \sur{Liu}}
	
\author[b]{\fnm{Boris}\sur{S. Mordukhovich}}\affil[a]
{\orgdiv{Department of Industrial and Systems Engineering}, \orgname{Wayne State University}, \city{Detroit}, \postcode{48201}, \state{Michigan}, \country{USA}}

\affil[b]
{\orgdiv{Department of Mathematics and the Institute for AI and Data Science}, \orgname{Wayne State University}, \city{Detroit}, \postcode{48201}, \state{Michigan}, \country{USA}}

\abstract{

\begin{abstract}  
	
Natural and human-made disasters can cause severe devastation and claim thousands of lives worldwide. Therefore, developing efficient methods for disaster response and management is a critical task for relief teams. One of the most essential components of effective response is the rapid collection of information about affected areas, damages, and victims. More data translates into better coordination, faster rescue operations, and ultimately, more lives saved. However, in some disasters, such as earthquakes, the communication infrastructure is often partially or completely destroyed, making it extremely difficult for victims to send distress signals and for rescue teams to locate and assist them in time.
	
{\em Unmanned Aerial Vehicles} (UAVs) have emerged as valuable tools in such scenarios. In particular, a fleet of UAVs can be dispatched from a mobile station to the affected area to facilitate data collection and establish temporary communication networks. Nevertheless, real-world deployment of UAVs faces several challenges, with adverse weather conditions--especially wind--being among the most significant. To address such challenges, we develop a novel mathematical framework to determine the optimal location of a mobile UAV station while explicitly accounting for the heterogeneity of the UAVs and the effect of wind. Our approach extends well-known single-facility location models by incorporating heterogeneous dynamic sets that represent different UAV operating speeds. In particular, we extend the classical Fermat-Torricelli and Sylvester problems to introduce the generalized  \emph{Sylvester-Fermat-Torricelli (SFT) model} and its remarkable specifications that capture complex factors such as wind influence, UAV heterogeneity, and back-and-forth motion within a unified framework. The proposed framework enhances the practicality of UAV-based disaster response planning by accounting for major real-world factors including wind and UAV heterogeneity. Experimental results demonstrate that our developments can reduce wasted operational time by up to 84\% while making post-disaster missions significantly more efficient and effective in practical settings.
\end{abstract}}

\keywords{location science, Facility location problems, Unmanned aerial vehicle (UAV). Generalized Sylvester-Fermat-Torricelli models, Convex analysis, Nonsmooth optimization}

\maketitle
\vspace*{-0.25in}
\section{Introduction}\label{sec:intro}
\subsection{State-of-the-Art and Motivations}

Natural and human-made disasters worldwide can have devastating consequences leading to significant loss of life and extensive destruction, often resulting in substantial economic and social costs. For instance, in 2023, the Emergency Events Database (EM-DAT) recorded 399 disasters worldwide causing 86,473 fatalities and affecting approximately 93.1 million people \parencite{CRED2024}. Given the significant costs associated with disasters, disaster management has become a critical field focused on minimizing these costs. Effective disaster management strategies can be categorized into {\em pre-disaster} activities (prevention and preparedness) and {\em post-disaster} (response and recovery) ones \parencite{erdelj2017wireless}.

In post-disaster actions, one of the most critical tasks is {\em data collection}. This data includes the locations of victims trapped under debris, the extent of road destruction and accessibility, and the damage levels across different parts of the region of interest. In this context, {\em UAVs} offer significant advantages, including rapid area surveying, ease of deployment, low-cost maintenance, and the ability to reach inaccessible locations \parencite{lyu2023unmanned}, making them highly effective for data gathering. Moreover, since natural disasters can damage or render network infrastructure inoperable \parencite{erdelj2017wireless}, UAVs can also serve as a vital tool for establishing temporary communication networks. Therefore, several studies have explored the use of UAVs to provide wireless or Delay-Tolerant Networks (DTNs) in such scenarios \parencite{shakhatreh2019uavs, matracia2021topological, 8490824, erdelj2017wireless,  khan2022emerging,erdelj2017help, wang2021disaster, 8641424}.

While UAVs can be valuable for providing network coverage, their {\em main drawback} is limited energy capacity requiring them to return periodically for recharging \parencite{shakhatreh2019uavs}. Also, unlike other vehicles, UAVs are susceptible to wind due to their lighter weight, smaller size, lower flight altitude, and
lower speed \parencite{Gianfelice2022}. These limitations create significant challenges in deploying UAVs for post-disaster network provisioning. \textcite{erdelj2017wireless} proposed a mobile first-response UAV station with a fleet of heterogeneous UAVs that has a long-distance communication antenna, an electricity generator, and a system for automatic UAV battery recharging. Motivated by this idea, we aim here to develop a {\em reliable mathematical model} for determining the {\em optimal location} of a mobile UAV station, considering UAV {\em heterogeneity} and the impact of {\em wind}, which are common factors in real-world scenarios.

Such a model can be regarded as a {\em single-facility location problem} previously studied in the context of drone-assisted delivery, where drones launch from a truck, deliver parcels, and subsequently return to the truck. There are studies \parencite{mourelo2016optimization,chang2018optimal,salama2020joint, betti2023dispatching} in the literature that examine the location of launch points (location of truck) for drones in a delivery network, where a truck and multiple drones operate in tandem. \textcite{scott2017drone} considered a tandem delivery strategy like \textcite{mourelo2016optimization} to use in the healthcare application. However,  to the best of our knowledge, little attention has been given to addressing the {\em impact of wind on drone performance}. This factor cannot be ignored as it causes delays in delivery time and increases energy consumption.  We also tried to consider the effect of wind in other applications that are search and rescue \parencite{kazemdehbashi2025algorithm}, and package delivery \parencite{liu2023routing}.

The paper by \textcite{dukkanci2024facility} provides a comprehensive review of drone-assisted delivery and highlights future research directions. It emphasizes that weather conditions--particularly wind--should be accurately accounted for in future studies. Moreover, as noted in \textcite{dukkanci2024facility_EJOR}, existing drone-based research has not fully leveraged advancements from non-drone-related studies including the use of heterogeneous fleets. These observations underscore a broader challenge: modern facility location problems in UAV applications demand frameworks that can simultaneously address {\em environmental uncertainties} and fleet {\em heterogeneity}. To bridge this gap, our method explicitly considers a heterogeneous fleet of UAVs and the effect of wind while making it more aligned with real-world applications. This forms yet another aspect of our motivation for this work.

\subsection{Contributions}

To address the challenges discussed in the previous section, namely the effects of wind and UAV heterogeneity, we focus on {\em single-facility models} of {\em location science}. Problems of this type have strong theoretical foundations and have been studied using various techniques. To provide a clear foundation for our approach, we first present a brief overview of the classical problems in this area. The French mathematician Pierre de Fermat (1601–1665) proposed an optimization problem that seeks to identify a point in the plane that minimizes the total distance to three other plane points. This problem was solved by the Italian physicist Augusto Torricelli (1608–1647) and is now known as the \emph{Fermat-Torricelli problem}. In the 19th century, the English mathematician  James Joseph Sylvester (1814–1897) introduced the \emph{smallest enclosing circle problem}, which seeks to determine the smallest circle capable of enclosing a given set of points in a plane \parencite{sylvester1857question}. This problem is formulated as a minimax location problem and has been extended in several ways in the literature. Over the centuries, these  classical problems have continued to attract attention of researchers and practitioners for their mathematical elegance and significant practical applications. Various generalizations of Fermat-Torricelli, Sylvester and related problems (Steinitz, Weber, etc.) have been introduced and studied in location science and practical areas by using different approaches and methods; see, e.g., \parencite{borwein1990,boltyanski1999,kuhn1976,plastria2009asymmetric,wesolowsky1993,weiszfeld37,martini2002fermat}, and the references therein.

More recent years have witnessed the development of {\em novel approaches} to location science models based on {\em variational} ideas with the usage of {\em generalized differential} tools of {\em convex and variational analysis}. To the best of our knowledge, such an approach was first suggested in \parencite{mordukhovich2011applications} for the newly introduced generalized Fermat-Torricelli problem and then was largely developed in \parencite{reich2022,kupitz2013fermat,mordukhovich2019fermat,mordukhovich2013smallest,nam2014constructions,nam2013generalized,jahn2015minsum,nam2014nonsmooth,nam2017minimizing,mordukhovich2023easy} and other publications for various models of location science with algorithmic applications.\vspace*{0.05in}

The major contributions of our work, including theoretical advances and practical applications, are as follows.

\begin{itemize}

\item We introduce and investigate the generalized \emph{Sylvester-Fermat-Torricelli (SFT)} problem, a novel model of location science not previously explored in the literature. The generalized SFT problem belongs to the area of {\em minimax optimization} that incorporates {\em multiple Minkowski gauge} functions rather than a single one. This generalization provides sufficient flexibility to account for {\em UAV heterogeneity} and the {\em influence of wind} as a source of weather uncertainty. Furthermore, the model extends from individual target points to convex sets, which is particularly useful when UAVs are required to reach a target area rather than a specific point. Finally, in this new minimax framework, the maximum of functions--each potentially representing the sum of several Minkowski gauge functions--is minimized. This approach enables modeling of {\em UAV round-trip movements}, a feature that is highly essential for realistic operational scenarios.

\item We further introduce and discuss an {\em extended version} of the generalized SFT model, as well as extended counterparts of the  generalized Fermat–Torricelli and Sylvester problems, by incorporating two types of Minkowski gauge functions: the {\em set-based Minkowski gauge} (\ref{eq:set-based-minkowski-definition}) and the {\em maximal set-based Minkowski gauge} (\ref{eq:MSMG_function}). These extensions enable us modeling scenarios in which UAVs enter a target area from the point {\em closest to a facility} and exit from the point {\em farthest from it}. In Section~\ref{sec:application}, the practical applicability of the extended formulations in real-world operations is demonstrated with providing extensive numerical experiments.

\item We develop a {\em single facility location model} for disaster relief by using UAVs that explicitly accounts for wind effects and UAV fleet heterogeneity, which therefore enhance its applicability to real-world scenarios.
\end{itemize}

The reminder of the paper is organized as follows. In Section~\ref{sec:problem}, we first review some background material from convex analysis needed below and then formulate the novel generalize Sylvester-Fermat-Torricelli model with its major special cases. Section~\ref{sec:properties_of_SMG} is devoted to the introduction and a detailed study of the set-based Minkowski gauge function, which plays a crucial role in the paper. Section~\ref{sec:existence-uniqueness} verifies well-posedness issues (existence and uniqueness of optimal solutions) for the generalized SFT. In Section~\ref{weight}, we discuss some extended and weight versions of the generalized SFT. Section~\ref{sec:application} contains the major applications of our theoretical developments to disaster relief operations by using UAVs with presenting the results of numerical experiments. In Section~\ref{sec:conc}, we summarize the main achievements of the paper and discuss some topics of our future research. The Appendix (Section~\ref{appendix}) presents some additional material of its own interest, which is used in the proofs of the obtained results.

\section{Basic Definitions and Problem Formulation}\label{sec:problem}

This section is started with reviewing key concepts and tools of convex analysis, which serve as a foundational framework for addressing the major challenges posed by the {\em nonsmooth} nature of our model introduced in (\ref{eq:SFT}). These tools will be instrumental in revealing and analyzing variational properties of the model, as well as in developing effective solution approaches. Then essential definitions and the problem formulation are presented and discussed. We also emphasize the key features introduced in our model, which extend the capabilities beyond those of existing models in location science.
These extensions allow us to address more complex scenarios that are closer to real-world applications.

\subsection{Tools of Convex Analysis}

In this subsection, we present key definitions and results of convex analysis that are essential for the main results of the paper; see, e.g., the books \parencite{rockafellar2015convex,mordukhovich2022} for more details, additional material, and references. Throughout the paper, we use the standard notation from the aforementioned books. Recall that the symbol $\Omega_1\subset\Omega_2$  means that the set $\Omega_1$ is smaller than or equal to the set $\Omega_2$. 

A nonempty subset $\Omega$ of $\mathbb{R}^q$ is {\em convex} if $\lambda x+(1-\lambda)y \in \Omega$ for all $x,y \in \Omega$ and $\lambda \in (0,1)$. An {\em extended-real-valued function} \( f: \Omega \to \bar{\mathbb{R}} := (-\infty, \infty] \) is {\em convex} on a convex set \( \Omega \) if
\begin{equation}\label{eq:convexity}
f\big(\lambda x+(1-\lambda)y\big)\le\lambda f(x)+(1-\lambda)f(y)\;\mbox{ for all
}\;x, y\in \Omega\;\mbox{ and }\;\lambda\in (0,1).
\end{equation}
$f$ is {\em strictly convex} if the inequality in (\ref{eq:convexity}) becomes strict for $x \neq y$.

Given a nonempty convex set $\Omega \subset \mathbb{R}^q$ and a point $\bar{x}\in \Omega$, the \emph{normal cone} to $\Omega$ at $\bar{x}$ is 
\begin{equation}\label{eq:normal-cone}
N(\bar{x};\Omega):=\{v \in \mathbb{R}^q \  |\ \langle v,x-\bar{x} \rangle \le0 \ \textnormal{  for all }\  x\in \Omega\}.
\end{equation}
If $\bar{x} \notin \Omega$, we have $N(\bar{x};\Omega):=\emptyset$ by convention.\vspace*{0.05in}

The convexity of functions is preserved under important operations.\vspace*{0.05in}

\begin{proposition} \label{prop:convexity_of_sum_funct_and_max_funct}
Let $f,f_i : \mathbb{R}^q \to \overline{\mathbb{R}}$ be convex functions for all $i = 1, \ldots, m$. Then the following functions are also convex:
\begin{enumerate}
\item The scaled function $\alpha f$ for any $\alpha \ge 0$.
\item The sum function $\sum_{i=1}^m f_i$.
\item The maximum function $\max_{1 \le i \le m} f_i$.
\end{enumerate}
\end{proposition}

The next result shows that the maximum of strictly convex functions is strictly convex.\vspace*{0.05in}

\begin{proposition} \label{lemma:max_strictly_convex}
Let \( f_i : \mathbb{R}^q \to \overline{\mathbb{R}} \) be strictly convex function for all \( i = 1, \dots, m \). Then the maximum function \( f(x): =\max_{1\le i\le m}f_i(x),\;x\in\mathbb{R}^q \), is strictly convex on $\mathbb{R}^q$.
\end{proposition}
\begin{proof} Consider first the case where $m=2$, i.e., \( f(x)=\max\{f_1(x),f_2)(x)\}\). Fix \( x, y \in \mathbb{R}^q \), \( x \neq y \), and \( \lambda \in (0,1) \) and suppose that $f(\lambda x + (1 - \lambda) y) = f_1(\lambda x + (1 - \lambda) y)$. It follows from the strict convexity of \( f_1 \) that
\[
f_1(\lambda x + (1 - \lambda) y) < \lambda f_1(x) + (1 - \lambda) f_1(y) \le  \lambda f(x) + (1 - \lambda) f(y),
\]
which tells us therefore that
\[
f(\lambda x + (1 - \lambda) y) < \lambda f(x) + (1 - \lambda) f(y).
\]
The remaining case where
$f(\lambda x + (1 - \lambda) y) = f_2(\lambda x + (1 - \lambda) y)$ is treated similarly. For an arbitrary $m\in\mathbb{N}$, we proceed by induction.
\end{proof}

One of the central concepts in convex analysis is the {\em subdifferential} of convex functions, which extends the classical derivative/gradient notion to nondifferentiable (nonsmooth) functions. Given an extended-real-valued convex function \( f : \mathbb{R}^q \to \overline{\mathbb{R}} \) finite at $\bar x$ (i.e., with $\bar x\in{\rm dom}\,f$), a {\em subgradient} \( v \in \mathbb{R}^q \) of $f$ at $\bar x$ is defined by 
\begin{equation} \label{convex subdifferential}
\langle v, x - \bar{x} \rangle \leq f(x) - f(\bar{x}) \quad \text{for all }\;x \in \mathbb{R}^q.
\end{equation}
The collection of all subgradients $v$ of \( f \) at \( \bar{x} \) is called the \emph{subdifferential} of \( f \) at this point and is denoted by \( \partial f(\bar{x})\). It is well known that the subdifferential of any convex function is nonempty at every point where the function is continuous.\vspace*{0.03in} 

The subdifferential of convex functions enjoys comprehensive {\em calculus} with respect to operations that keep the convexity. For the reader's convenience, we present the two fundamental calculus results used in what follows. The first result is the {\em subdifferential maximum rule}.\vspace*{0.05in}

\begin{theorem}\label{theorem:max_rule}
Let $f_i:\mathbb{R}^q \to \overline{\mathbb{R}}$, $i=1,\ldots,m$, be convex functions with $\bar{x} \in \bigcap_{i=1}^m \textnormal{dom}\ f_i$. Assume that all $f_i$ are continuous at $\bar x$ and consider the maximum function
\begin{equation*}
f(x)=\ \max\big\{f_i(x)\;\big|\; i=1,\dots,m\big\}, \quad x\in\mathbb{R}^q. 
\end{equation*}
Then the subdifferential of $f$ at $\bar x$ is calculated by
\[
\partial f(\bar{x})=\ \textnormal{co }\Big\{\bigcup_{i \in I(\bar{x})}  \partial f_i(\bar{x})\Big\},
\]
where $I(\bar{x}):=\{i=1,\dots,m \ |\ f(\bar{x})=f_i(\bar{x})\}$ is the collection of active indexes at $\bar{x}$, and where the symbol ``{\rm co}" signifies the convex hull if the set in question.
\end{theorem}\vspace*{0.05in}

Theorem~\ref{theorem:max_rule} tells us that the subdifferential of $f$ at $\bar{x}$ is the convex hull of the subgradients of the functions that achieve the maximum at $\bar{x}$. Recall that the convex hull of a set $\Omega\subset\mathbb{R}^n$ is 
\begin{equation*}
\operatorname{co} \Omega := \left\{ \sum_{i=1}^{k} \lambda_i x_i \,\middle|\, x_i \in \Omega,\ \lambda_i \geq 0,\ \sum_{i=1}^{k} \lambda_i = 1,\ k \in \mathbb{N} \right\},
\end{equation*}
where $k$ can be bounded by $n+1$ due to the classical Carath\'eodory theorem.\vspace*{0.05in}

The following calculus result used below is the fundamental {\em subdifferential sum rule} known also as the {\em Moreau-Rockafellar theorem}. \vspace*{0.05in}

\begin{theorem}\label{theorem:sum_rule_subdifferential}
Let $f_i:\mathbb{R}^q \rightarrow \mathbb{\overline{R}}$ for $i=1,\dots,m$ be convex functions such that all but one of $f_i$ are continuous at $\bar{x} \in \bigcap_{i=1}^m \textnormal{dom}\ f_i$. Then  we have
\[
\partial\Big(\sum_{i=1}^{m}f_i\Big)(\bar{x})=\sum_{i=1}^{m}\partial f_i(\bar{x}).
\]
\end{theorem}

\subsection{Problem Formulation}

Let $F$ be a nonempty, closed, and convex subset of the space $\mathbb{R} ^q$. The \emph{Minkowski gauge function} associated with the set $F$, is defined by
\begin{equation}\label{eq:minkowski_definition}
\rho_F(x):=\inf\big\{t\ge0 \ \big| \ x \in tF\big\}, \ \  x\in\mathbb{R}^q.
\end{equation}
The Minkowski gauge \eqref{eq:minkowski_definition}, which is common in the literature, plays a key role in convex analysis and its applications; see, e.g., \parencite{mordukhovich2023easy}. However, for our purposes in this paper, we require a more flexible function as defined below.

Given a nonempty set $\Omega \subset \mathbb{R}^q$ in addition to $F$, we define the \emph{set-based Minkowski gauge function} associated with the sets $F$ and $\Omega$ by
\begin{equation}\label{eq:set-based-minkowski-definition}
\rho_F^{\Omega}(x):= \inf\big\{t \ge 0 \ \big| \ x\in \Omega+tF\big\},
\end{equation}
where $\Omega+tF:=\{\omega+tf \ |\ \textnormal{for all }\omega\in \Omega \textnormal{ and } f\in F \}$. In \eqref{eq:set-based-minkowski-definition}, the sets $F$ and $\Omega$ are referred to as the (constant) \emph{dynamic set} and the \emph{reference set}, respectively. In simple terms, if we view \( F \) as a dynamic set representing possible speeds or directions of motion, then \( \rho_F^{\Omega}(x) \) represents the {\em minimal time} required to reach \( x \) starting from at least one point \( \omega \in \Omega \). Moreover, \( \rho_F \) can be interpreted in this context by noting that \( \rho_F(x) = \rho_F^{\{0\}}(x) \), {\color{black}see Fig. \ref{fig:SMG_function_example} for a graphical illustration}. 

\begin{figure}
    \centering
    \includegraphics[width=1\linewidth]{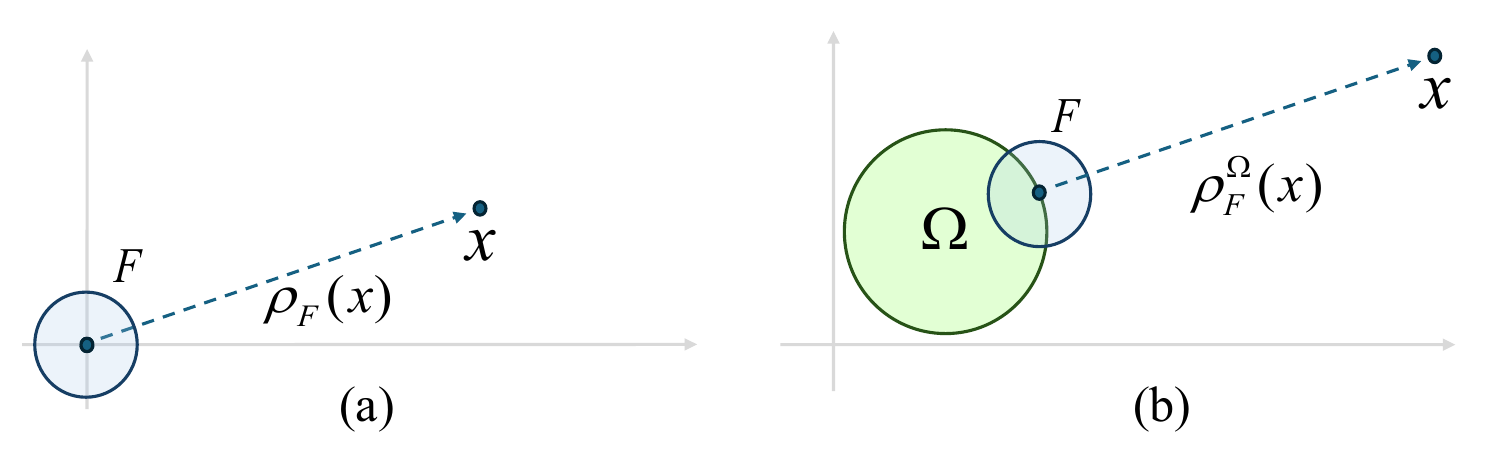}
    \caption{\textcolor{black}{(a) illustrates an example of the Minkowski gauge function (\ref{eq:minkowski_definition}) with a circle as its dynamic set. (b) shows an example of the set-based Minkowski gauge function (\ref{eq:set-based-minkowski-definition}), where both the reference set and the dynamic set are circles.}}
    \label{fig:SMG_function_example}
\end{figure}

With these preliminaries in place, we now proceed to formulate the proposed model. Let the index sets \( I_k \neq \emptyset \) for \( k = 1, \dots, n \) form a partition of the set \( I := \{1, \dots, m\} \), such that \( \bigcup_{k=1}^{n} I_k = I \) and \( I_k \cap I_l = \emptyset \) for all \( k \neq l \) with \( k, l \in \{1, \dots, n\} \). For $i \in I_k$ with $k=1,\dots n$, let \( F_i \subset \mathbb{R}^q \) be a compact and convex set containing the origin in its interior point (i.e., \( 0 \in \operatorname{int} F_i \)). Suppose also that \( \Omega_i \) and the constraint set \( \Omega_0 \) are nonempty, closed, and  convex in \( \mathbb{R}^q \). Then the proposed \emph{generalized Sylvester-Fermat-Torricelli (SFT)} problem is formulated as follows:
\begin{equation}\label{eq:SFT}
\min S(x)\;\mbox{ subject to }\;x\in \Omega_0,
\end{equation}
where the cost function in \eqref{eq:SFT} is defined by
\begin{equation}\label{eq:SFT-definition}
S(x):=\max\big\{\sum_{i \in I_k} \rho_{F_i}^{\Omega_i}(x) \  \big| \ k=1,,...,n \big\}.
\end{equation}
In the case where \( I_k = \{k\} \) and $n=m$, the novel generalized SFT problem is written as
\begin{equation}\label{eq:generalized-Syvester-Model}
\min Y(x)\;\mbox{ such that }\;x\in \Omega_0
\end{equation}
with the cost function given in the simple maximum form
\begin{equation}\label{eq:generalized-sylvester-equation}
Y(x):= \max\big\{\rho_{F_i}^{\Omega_i}(x)\ \big| \ i=1,\dots,n\big\}. 
\end{equation}
This extends the \emph{generalized Sylvester} problem studied in \parencite{mordukhovich2023easy}, where only one dynamic set $F$ is under consideration, and where $\rho^{\Omega_i}_F$ is replaced by the ``minimal time function" $T^F_{\Omega_i}$ that agrees with $\rho^{\Omega_i}_{-F}$; see Lemma~\ref{lemma:relationship_SMG_MinimalTime} in the Appendix. Furthermore, if \( k = 1 \) and $|I_1|=m$ for the cardinality of $I_1$, 
then  the proposed generalized SFT problem reduces to the \emph{generalized Fermat–Torricelli} one stated as follows:
\begin{equation}\label{eq:eq:Generalized-Fermat-Model}
\min T(x)\;\mbox{ subject to }\;x\in \Omega_0
\end{equation}
with the cost function given in the summation form
\begin{equation}\label{eq:Generalized-Fermat-equation}
T(x):= \sum_{i=1}^{m} \rho_{F_i}^{\Omega_i}(x).
\end{equation}
The latter extends the previous version explored in \parencite{mordukhovich2023easy}, where  only one dynamic set $F$ is present in \eqref{eq:Generalized-Fermat-equation}, and where $\rho_F^{\Omega_i}$ is replaced by $T^F_{\Omega_i}$.\vspace*{0.03in}

This paper focuses on the generalized SFT problem \eqref{eq:SFT}, which encompasses both generalized Sylvester and Fermat–Torricelli problems as special cases. Let us emphasize that, unlike the previously considered version of the generalized Sylvester problem with a single identical dynamic set, our new setting in \eqref{eq:generalized-Syvester-Model}, \eqref{eq:generalized-sylvester-equation} accommodates multiple dynamic sets important in practical applications. Moreover, another key generalization is using the sum $\sum_{i \in I_k} \rho_{F_i}^{\Omega_i}(x)$ instead of a single term $\rho_{F}^{\Omega}(x)$, which enables us to model more complex movement patterns such as, e.g., round-trip UAV flights under wind conditions between two points.

\section{Properties of the Set-Based Minkowski Gauge Function}\label{sec:properties_of_SMG}

In this section, we establish some key properties of the set-based Minkowski gauge function that include the following issues:

\begin{enumerate}
\item[$\bullet$]  The relationship between the functions $\rho_F^\Omega$ and $\rho_F$, which serves as a foundational tool for deriving further properties and facilitating the computation of $\rho_F^\Omega$.

\item[$\bullet$] The convexity of $\rho_F^\Omega$, which enables us to utilize tools of convex analysis.

\item[$\bullet$] The Lipschitz continuity of $\rho_F^\Omega$, which implies that its subdifferential is bounded.

\item[$\bullet$] Characterizing the subdifferential of $\rho_F^\Omega$ in both settings of in-set points $x \in \Omega$ and out-of-set points $x \notin \Omega$.

\item[$\bullet$] Deriving an efficient upper estimate for the subdifferential of $\rho_F^\Omega$.
\end{enumerate}
In what follows, we use the {\em support function}
\begin{equation}\label{eq:support_function}
\sigma_F(v):= \sup\big\{\langle v,f \rangle \  \big|\  f \in F\big\},\  v\in \mathbb{R}^q,
\end{equation}
associated with  $F\subset\mathbb{R}^q$, and the {\em \( r \)-enlargement} of  \( \Omega \subset \mathbb{R}^q \) defined by 
\begin{equation}\label{eq:r_enlarge_set_based_minkowski}
\Omega_r :=\big\{ x \in \mathbb{R}^q\;\big|\;\rho_F^\Omega(x) \leq r \big\}\; \text{ for any  }\  r > 0.
\end{equation}

The first proposition reveals the precise relationship between the set-based Minkowsi gauge function \eqref{eq:set-based-minkowski-definition} and the classical Minkowski gauge  \eqref{eq:minkowski_definition}.\vspace*{0.05in}

\begin{proposition}\label{prop:relationship_between_set_based_and_Minkowski} Let $F \subset \mathbb{R}^q$ be a nonempty, closed, and convex set. Given an nonempty reference set $\Omega \subset \mathbb{R}^q$, we have the relationship
\begin{equation}\label{eq:set-based-mikowski-function-definition2}
\rho_F^{\Omega}(x)=\inf_{\omega \in \Omega} \rho_F(x-\omega).
\end{equation}
\end{proposition}
\begin{proof}
Starting with the case where \(\rho_F^{\Omega}(x) < \infty\), pick \(\epsilon > 0\) and denote  \(p:=\rho_F^{\Omega}(x)\). Choose further \(t\) such that \(p \le t < p + \epsilon\) and \(x \in \Omega + tF\). Then there exists \(\omega \in \Omega\) with \(x - \omega \in tF\), which yields \(\rho_F(x - \omega) \le t < p + \epsilon\). Since \(\rho_F^{\Omega}(x) = p\), it follows that
\[
\inf_{\omega \in \Omega} \rho_F(x - \omega) \le \rho_F(x - \omega) < \rho_F^{\Omega}(x) + \epsilon.
\]
Passing to the limit as \(\epsilon\downarrow 0\) gives us the estimate
\[
\inf_{\omega \in \Omega} \rho_F(x - \omega) \le \rho_F^{\Omega}(x).
\]
To verify  the reverse inequality, denote \(q:=\inf_{\omega \in \Omega} \rho_F(x - \omega)\), pick \(\alpha > 0\), and choose $d$ and $\omega$ such that \(q\le d < q + \alpha\) and that $x-\omega \in dF$. This leads us to \(x \in \omega + dF \subset \Omega + dF\), and hence \(\rho_F^{\Omega}(x) \le d < q + \alpha\). Passing to the limit as \(\alpha \to 0\), we arrive at 
\[
\rho_F^{\Omega}(x) \le \inf_{\omega \in \Omega} \rho_F(x - \omega).
\]
Combining both inequalities justifies the claimed representation \eqref{eq:set-based-mikowski-function-definition2} when  \(\rho_F^{\Omega}(x) < \infty\).

It remains to consider the case where \(\rho_F^{\Omega}(x) = \infty\). Then we get \(\{t \ge 0 \mid x \in \Omega + tF\} = \emptyset\). This implies that \(\{t \ge 0 \mid x - \omega \in tF\} = \emptyset\) for every \(\omega \in \Omega\) and so \(\rho_F(x - \omega) = \infty\). Therefore,
\[
\inf_{\omega \in \Omega} \rho_F(x - \omega) = \infty = \rho_F^{\Omega}(x),
\]
which completes the proof of the proposition.
\end{proof}

The next proposition addresses the convexity of the set-based Minkowski gauge function.\vspace*{0.05in}

\begin{proposition}\label{prop:convexity_of_set_based_minkowski}
Let $F\subset \mathbb{R}^q$ be a nonempty, closed, and convex set. If the reference set $\Omega \subset \mathbb{R}^q$ is nonempty and convex, then $\rho_F^{\Omega}$ in {\rm(\ref{eq:set-based-minkowski-definition})} is a convex function.
\end{proposition}
\begin{proof}
Fix an arbitrary number $\epsilon>0$ and select $t_1$ and $t_2$ such that $\rho_F^{\Omega}(x_i)\le t_i<\rho_F^{\Omega}(x_i)+\epsilon$ and $x_i \in \Omega+t_if_i$ where $f_i \in F$ as $i=1,2$. By $x_i-t_if_i\in \Omega$ and the convexity of $\Omega$, we have $\lambda(x_1-t_1f_1)+(1-\lambda)(x_2-t_2f_2)\in \Omega$ for any  $\lambda \in (0,1)$. Hence
\[\lambda x_1+(1-\lambda) x_2 \in \Omega +\lambda t_1 f_1+(1-\lambda) t_2 f_2 \subset \Omega + \lambda t_1 F +(1-\lambda) t_2 F.\]
The convexity of $F$ tells us that $\lambda t_1 F +(1-\lambda) t_2 F=(\lambda t_1+ (1-\lambda)t_2)F$,  and so we have
\[\lambda x_1+(1-\lambda)x_2 \in \Omega+(\lambda t_1+(1-\lambda)t_2)F\  \Longrightarrow\  \rho_F^{\Omega}(\lambda x_1+(1-\lambda)x_2) \le \lambda t_1+(1-\lambda) t_2.\]
Since $\epsilon>0$ was chosen arbitrarily, this yields the implication
\begin{equation*}
\begin{array}{ll}
&\lambda t_1+(1-\lambda)t_2 < \lambda \rho_F^{\Omega}(x_1)+(1-\lambda)\rho_F^{\Omega}(x_2)+\epsilon\\
&\Longrightarrow \rho_F^{\Omega}(\lambda x_1+(1-\lambda)x_2) \le \lambda \rho_F^{\Omega}(x_1)+(1-\lambda)\rho_F^{\Omega}(x_2),
\end{array}
\end{equation*}
which readily justifies the claimed convexity of the set-based Minkowski gauge $\rho_F^{\Omega}$.
\end{proof}

The following assertion establishes the global Lipschitz continuity of the function $\rho_F^{\Omega}$ on the entire space $\mathbb{R}^q$, which is a key property for our subsequent analysis.\vspace*{0.05in}

\begin{proposition}
\label{prop:proof_of_Lipshitz}
Let $F \subset \mathbb{R}^q$ be a closed and convex set with $0 \in{\rm int}\,F$. Then for any nonempty set $\Omega \subset \mathbb{R}^q$, we have that $\rho_F^{\Omega}(\cdot)$ is Lipschitz continuous on $\mathbb{R}^q$.
\end{proposition}
\begin{proof}
It is well known that the Minkowski gauge $\rho_F$ is subadditive and positively homogeneous; see, e.g., \parencite[Theorem~6.4]{mordukhovich2023easy}. Therefore, for any $x, y \in \mathbb{R}^q$, we get by using Proposition~\ref{prop:relationship_between_set_based_and_Minkowski} that
\begin{align*}
& \rho_F(x-\omega) \le \rho_F(x-y) + \rho_F(y-\omega)  \quad \text{for all } \omega \in \Omega, \\
& \Longrightarrow \inf_{\omega \in \Omega} \rho_F(x-\omega) \le \rho_F(x-y) + \inf_{\omega \in \Omega}\rho_F(y-\omega)\\
& \Longrightarrow  \inf_{\omega \in \Omega} \rho_F(x-\omega) - \inf_{\omega \in \Omega}\rho_F(y-\omega) \le \rho_F(x-y) \\
& \Longrightarrow  \rho_F^{\Omega}(x) -\rho_F^{\Omega}(y) \le \rho_F(x-y).
\end{align*}
The Lipschitz continuity of \( \rho_F \) by \parencite[Proposition~6.18]{mordukhovich2023easy} gives us a constant \( l \geq 0 \) such that \( \rho_F(x) \leq l \|x\| \) whenever $x \in \mathbb{R}^q$. Therefore,
\[
\rho_F^{\Omega}(x) -\rho_F^{\Omega}(y) \le \rho_F(x-y)\le l \|x-y\|,\]
which verifies the global Lipschitz continuity of the set-based Minkowski gauge. 
\end{proof}

Now we proceed with calculating the subdifferential of the function $\rho_F^{\Omega}$ at any point $\bar x\in\mathbb{R}^q$ starting with the {\em in-set} case $\bar x\in\Omega$.\vspace*{0.05in}

\begin{theorem}\label{prop:subdifferential_1}
Let $F \subset \mathbb{R}^q$ be compact and convex with $0 \in \textnormal{int}F$, and let $\Omega$ be nonempty and convex subset of $\mathbb{R}^q$. Then the subdifferential of \( \rho_F^{\Omega}(\cdot) \) at \( \bar{x} \in \Omega \) is calculated by
\begin{equation}\label{in-set}
\partial \rho_F^{\Omega}(\bar{x})=N(\bar{x};\Omega)\cap \mathcal{V},
\end{equation}
where \( \mathcal{V} \) is defined via the support function \eqref{eq:support_function} as
$$
\mathcal{V} := \big\{ v \in \mathbb{R}^q\;\big|\;\sigma_F(v) \leq 1\big\}.
$$
\end{theorem}
\begin{proof}
Fix any $v \in \partial \rho_F^\Omega (\bar{x})$. Then for $x \in \Omega$ with $\rho_F^\Omega(x)=0$, we get the conditions
\[
\langle v,x-\bar{x} \rangle \le \rho_F^\Omega(x)-\rho_F^\Omega(\bar{x})=0,
\]
which yield $v\in N(\bar{x};\Omega)$ by the normal cone definition \eqref{eq:normal-cone}. Fix $x \not \in \Omega$ and find $t>0$. $f \in F$ with $\bar{x}+tf=x$, which means that $\rho_F^\Omega(x)\le t$. Then we have the inequalities
\[
\langle v,\bar{x}+tf-\bar{x} \rangle\le\rho_F^\Omega(x)-0 \le t
\]
and thus arrive at $\sigma_F(v)\le 1$ and therefore justify the inclusion ``$\subset$" in \eqref{in-set}.

To verify the reverse inclusion in \eqref{in-set}, fix any $v \in N(\bar{x};\Omega)\cap \mathcal{V}$ and pick an arbitrary number $\epsilon>0$. Given $x \in \mathbb{R}^q$, select $t>0$ such that $\rho_F^\Omega(x)\le t<\rho_F^\Omega(x)+\epsilon$. Then find $\omega \in \Omega$ and $f \in F$ such that $\omega+tf=x$. Therefore, we get
\[
\langle v,\omega+tf-\bar{x} \rangle = \langle v,\omega - \bar{x} \rangle + t \langle v,f \rangle \le t < \rho_F^\Omega(x)+\epsilon=  \rho_F^\Omega(x)-\rho_F^\Omega(\bar{x})+\epsilon.\]
Passing to the limit therein as $\epsilon\downarrow 0$ gives us 
$\langle v,x-\bar{x} \rangle \le \rho_F^{\Omega}(x)-\rho_F^{\Omega}(\bar{x})$, which justifies the inclusion ``$\supset$" in \eqref{in-set} and thus completes the proof of the theorem.
\end{proof}

The next theorem provides the subdifferential calculation for $\rho_F^{\Omega}$ at {\em out-of-set} points $\bar x\notin\Omega$. \vspace*{0.05in}

\begin{theorem}\label{prop:subdifferential2} Let $F \subset 
\mathbb{R}^q$ be a compact and convex set with $0 \in \textnormal{int}F$, let $\Omega \subset \mathbb{R}^q$ be nonempty and convex, and let $\bar x\notin\Omega$. Denote $r:=\rho_F^{\Omega}(\bar{x})$ and consider the $r$-enlargement \eqref{eq:r_enlarge_set_based_minkowski}. Then the subdifferential of \( \rho_F^{\Omega} \) at \( \bar{x} \) is calculated by
\begin{equation}\label{out}
\partial \rho_F^\Omega(\bar{x})=N(\bar{x};\Omega_r)\cap \Upsilon\ \ \textnormal{with} \ \ \Upsilon:=\{v \in \mathbb{R}^q \ | \ \sigma_F(v)=1\}
\end{equation}
\end{theorem}
\begin{proof}
Fix \( v \in \partial \rho_F^{\Omega}(\bar{x}) \) and consider the following three cases.

{\bf(a)} Suppose that \( x \notin \Omega_r \). Then  there exist \( t > 0 \) and \( f \in F \) such that \( \bar{x} + t f = x \), which yields \( \rho_F^{\Omega}(x) \leq t + r \) by Lemma~\ref{lemma:minkowski_estimate} in the Appendix. This brings us to the conditions
\[
\langle v, \bar{x} + t f - \bar{x} \rangle = \langle v, t f \rangle \leq \rho_F^{\Omega}(x) - \rho_F^{\Omega}(\bar{x}) \leq t + r - r = t,
\]
which imply in turn that \( \langle v, f \rangle \leq 1 \), and thus \( \sigma_F(v) \leq 1 \).

{\bf(b)} Supposing that \( x \in \Omega_r \), we get
\[
\langle v, x - \bar{x} \rangle \leq \rho_F^{\Omega}(x) - \rho_F^{\Omega}(\bar{x}) \leq 0,
\]
which tells us that \( v \in N(\bar{x}, \Omega_r) \).

{\bf(c)} Consider finally the case where \( x \in \Omega \). Pick any \( \epsilon > 0 \) and find  \( x \in \Omega \) and \( f \in F \) such that \( x + d f = \bar{x} \), where \( r \leq d < r + \epsilon \). Then we have the implications
\begin{equation*}
\begin{array}{ll}
& \langle v,x-\bar{x} \rangle=\langle v,\bar{x}-df-\bar{x} \rangle \le \rho_F^{\Omega}(x)-\rho_F^{\Omega}(\bar{x}) \\
& \Longrightarrow r \le  \langle v, df \rangle < (r+\epsilon) \langle v,f \rangle\\
& \Longrightarrow\displaystyle\frac{r}{r+\epsilon} < \langle v,f \rangle \le \sigma_F(v).
\end{array}
\end{equation*}
Passing there to the limit as\( \epsilon\downarrow 0 \) gives us \( \sigma_F(v) \geq 1 \), and hence we conclude that
\[
v \in N(\bar{x}; \Omega_r) \cap \left\{ v \in \mathbb{R}^q \mid \sigma_F(v) \leq 1 \right\} \cap \left\{ v \in \mathbb{R}^q \mid \sigma_F(v) \geq 1 \right\} = N(\bar{x}; \Omega_r) \cap \Upsilon,
\]
which justifies the inclusion ``$\subset$" in \eqref{out}.

To verify the opposite inclusion in \eqref{out}, pick any vector $v \in N(\bar{x};\Omega_r)\cap \Upsilon$. We aim to check the fulfillment of the inequality 
\begin{equation}\label{eq:eq_for_theorem_upsilon}
\langle v,x-\bar{x} \rangle \le \rho_F^{\Omega}(x)-\rho_F^{\Omega}(\bar{x}) \ \ \textnormal{for all} \ \ x\in \mathbb{R}^q.
\end{equation}
Applying Theorem~\ref{prop:subdifferential_1} tells us that $v \in \partial \rho_F^{\Omega_r}(\bar{x})$, which yields
\[
\langle v,x-\bar{x}\rangle \le \rho_F^{\Omega_r}(x)-\rho_F^{\Omega_r}(\bar{x}).
\]
If $x\notin \Omega_r$, we use Lemma~\ref{lemma:equality_mikowski_to_omega_r} from the Appendix and get
\[
\langle v,x-\bar{x}\rangle \le \rho_F^{\Omega_r}(x)=\rho_F^{\Omega}(x)-r=\rho_F^{\Omega}(x)-\rho_F^{\Omega}(\bar{x}),
\]
which readily ensures that $v\in \partial \rho_F^{\Omega}(\bar{x})$. In the other case, suppose that $x\in \Omega_r$, i.e., $\rho_F^{\Omega}(x)=p \le r$. Fix any $\epsilon>0$ and deduce from $\sigma_F(v)=1$ the existence of $f\in F$ such that $\langle v,f \rangle \ge 1-\epsilon$. Using Lemma~\ref{lemma:minkowski_estimate} from the Appendix, we obtain \( y \in \Omega_r \) with \( y = x + (r - p)f \). Since \( v \in N(\bar{x}; \Omega_r) \), it follows that \( \langle v, y - \bar{x} \rangle \leq 0 \). Therefore,
\begin{align*}
& \langle v,x-(p-r)f-\bar{x} \rangle=\langle v,x-\bar{x} \rangle-(p-r)\langle v,f \rangle \le 0 \\
& \Longrightarrow \langle v, x-\bar{x} \rangle \le (p-r)  \langle v,f \rangle \le (1-\epsilon)(p-r)=(1-\epsilon)(\rho_F^{\Omega}(x)-\rho_F^{\Omega}(\bar{x})).
\end{align*}
Passing to the limit as $\epsilon\downarrow 0$ gives us $v\in \partial\rho_F^{\Omega}(\bar{x})$. This shows that $N(\bar{x};\Omega_r)\cap \Upsilon \subset\partial\rho_F^{\Omega}(\bar{x})$ and thus completes the proof of the theorem.   \end{proof}

In what follows, we need yet another evaluation of the subdifferential of the set-based Minkowski gauge function at out-of-set points. Given a compact and convex set $F\subset\mathbb{R}^q$ with $0 \in \textnormal{int}\,F$, and given a nonempty, closed, and  convex set $\Omega\subset\mathbb{R}^q$, the \emph{generalized projection} from a point $x\in\mathbb{R}^q$ to the set $\Omega$ relative to $F$, denoted by $\Pi_F(x;\Omega)$, is defined by
\begin{equation}\label{eq:generalized_projection}
\Pi_F(x;\Omega):=\big\{\omega \in \Omega\ \big|\ \rho_F^\Omega (x)=\rho_F(x-w)\big\}.
\end{equation}
The nonemptiness of \( \Pi_F(x; \Omega) \) is established in Proposition~\ref{prop:generalized_projection_nonempty} from the Appendix.\vspace*{0.05in}

\begin{proposition}\label{prop:subdifferential_estimation} Let $F$ and $\Omega$ be the sets from definition \eqref{eq:generalized_projection} of the generalized projection under the assumptions imposed therein, and let $\bar x\in{\rm dom}\,\rho^\Omega_F$. Then there is $\bar{\omega} \in \Pi_F(\bar{x};\Omega)$ with 
\begin{equation}\label{dist-proj}
\partial \rho_F^\Omega(\bar{x}) \subset \partial \rho_F(\bar{x} - \bar{\omega}).
\end{equation}
If in addition we have $\bar{x}\in \textnormal{int(dom}\,\rho_F^{\Omega})$ and the function $\rho_F$ is differentiable at $\bar{x}-\bar{\omega}$, then the inclusion in \eqref{dist-proj} becomes the equality
\begin{equation}\label{dist-prof1}
\partial \rho_F^\Omega(\bar{x}) = \partial \rho_F(\bar{x} - \bar{\omega})=\{\nabla \rho_F(\bar{x}-\bar{\omega})\}.
\end{equation}
\end{proposition}
\begin{proof}
Picking \( v \in \partial \rho_F^\Omega(\bar{x}) \), we deduce from Proposition~\ref{prop:relationship_between_set_based_and_Minkowski} that
\[
\langle v, x-\bar{x} \rangle \le \rho_F^\Omega(x)-\rho_F^\Omega(\bar{x}) = \inf_{\omega \in \Omega} \rho_F(x-\omega)-\inf_{\omega \in \Omega} \rho_F(\bar{x}-\omega)\le \rho_F(x-\bar{\omega})-\rho_F(\bar{x}-\bar{\omega})
\]
which gives us $v \in \partial \rho_F(\bar{x}-\bar{\omega})$ and thus verifies the inclusion in \eqref{dist-proj}.

To prove the second part of the proposition, we deduce from the convexity  of  \( \rho_F^{\Omega} \) by Proposition~\ref{prop:convexity_of_set_based_minkowski} and basic convex analysis that $\partial\rho_F^{\Omega}(\bar{x})\ne\emptyset$ if \( \bar{x} \in \operatorname{int}(\operatorname{dom}\,\rho_F^{\Omega})\). Furthermore,  the subdifferential \(\partial \rho_F^\Omega(\bar{x} - \bar{\omega})\) reduces to a singleton if \(\rho_F\) is differentiable at \(\bar{x} - \bar{\omega}\). This readily leads us to \eqref{dist-prof1} and thus completes the proof of all the claimed statements.
\end{proof}

Alternative calculations of subgradients of the set-based Minkowski gauge functions at out-of-state points, which involve the generalized projection \eqref{eq:generalized_projection} and the normal cone to the $r$-enlargement \eqref{eq:r_enlarge_set_based_minkowski}, are given in Theorem~\ref{sub-proj}.\vspace*{0.03in}

The next proposition establishes the  precise relationship between \( \rho_F^{\Omega} \) and its multiplication by positive scalars while demonstrating how positive coefficients and weights can be incorporated into the dynamic set.\vspace*{0.05in}

\begin{proposition}\label{prop:positive_scaled_set_based_minkowski}
Let $F \subset \mathbb{R}^q$ be a nonempty, closed, and convex set, and let $\Omega \subset \mathbb{R}^q$ be a nonempty set. Then we have the relationship
\begin{equation}\label{scalar}
\lambda \rho^{\Omega}_F(x)=\rho^{\Omega}_{\frac{F}{\lambda}}(x)\;\mbox{ for any }\;\lambda>0.
\end{equation}
\end{proposition}
\begin{proof}
We begin with the case where $\rho_F^{\Omega}(x)<\infty$. Given $\epsilon > 0$, choose $t$ such that $\rho_F^{\Omega}(x) \le t < \rho_F^{\Omega}(x) + \epsilon$. Then there exist vectors $\omega \in \Omega$ and $f \in F$ with $x = \omega + t f$. Rewrite the latter in the form $x = \omega + \lambda t \cdot \frac{f}{\lambda}$ and get $x \in \Omega +(\lambda t) \frac{F}{\lambda}$, which implies in turn that $\rho_{\frac{F}{\lambda}}^{\Omega}(x) \le \lambda t < \lambda(\rho_F^{\Omega}(x) + \epsilon)$. 
Letting $\epsilon\downarrow 0$ gives us $\rho_{\frac{F}{\lambda}}^{\Omega}(x) \le \lambda \rho_F^{\Omega}(x)$.

To verify the reverse inequality in\eqref{scalar}, pick $\alpha > 0$ and choose a number $\beta$ such that 
$$
\rho_{\frac{F}{\lambda}}^{\Omega}(x) \le\beta < \rho_{\frac{F}{\lambda}}^{\Omega}(x) + \alpha.
$$
Then there are $\omega \in \Omega$ and $\lambda^{-1}f\in\lambda^{-1}F$ with $x = \omega + \lambda^{-1}\beta f$. This yields $x \in \Omega + \lambda^{-1}\beta F$, which leads us to the inequalities
$$
\rho_F^{\Omega}(x) \le\lambda^{-1}\beta <\lambda^{-1}\Big(\rho_{\frac{F}{\lambda}}^{\Omega}(x) + \alpha\Big),
$$
which yield $\lambda \rho_F^{\Omega}(x) < \rho_{\frac{F}{\lambda}}^{\Omega}(x) + \alpha$. Letting $\alpha\downarrow 0$ implies that $\lambda \rho_F^{\Omega}(x) \le \rho_{\frac{F}{\lambda}}^{\Omega}(x)$ and thus justifies the claimed relationship in \eqref{scalar}.

It remains to consider the case where $\rho_F^{\Omega}(x)=\infty$ and therefore $\{t\ge0 \ |\ x\in\Omega+tF\}=\emptyset$. The latter readily implies that $\{t\ge 0 \ | \ x\in \Omega+t\frac{F}{\lambda}\}=\emptyset$ and thus
\[\lambda\rho_F^{\Omega}(x)=\infty=\rho_{\frac{F}{\lambda}}^{\Omega}(x),\]
which completes the proof of the proposition.
\end{proof}

To proceed further, we introduce a new function important in the subsequent applications. Let $F \subset \mathbb{R} ^n$ be a nonempty, closed, bounded, and convex set, and let $\Omega \subset \mathbb{R}^q$ be nonempty and bounded. The \emph{maximal set-based Minkowski gauge} (MSMG) is defined by
\begin{equation}\label{eq:MSMG_function}
\bar{\rho}_F^{\Omega}(x):=\inf\big\{t\ge0 \  \big| \ x-\Omega \subset tF\big\},\quad x \in \mathbb{R}^q.
\end{equation}
In simple terms, the set \( x - \Omega \) includes all vectors that start from a point in \( \Omega \) and end in \( x \). If we view \( F \) as a dynamic set, then \( \bar{\rho}_F^{\Omega}(x)\) represents the {\em minimal time} required to reach \( x \) from the {\em farthest point} in \( \Omega \), \textcolor{black}{see Fig. \ref{fig:mdf-msmg}(a) as an example}. Lemma~\ref{lemma:relationship_MSMG_Maximal_Time_Function} from the Appendix provides the connection between \( \bar{\rho}_F^{\Omega}(x)\) and the {\em maximal time function} \(C_{\Omega}^F(x)\) defined in (\ref{eq:Maximal_time_function_definition}), The application of the MSMG function will be presented in Section~\ref{sec:application}. Now we derive a relationship between  \eqref{eq:MSMG_function} and the classical Minkowski gauge function $\rho_F$.\vspace*{0.05in}

\begin{proposition}\label{prop:relationship_MSMG_Minkowski} Let $F \subset \mathbb{R}^q$ be a compact and convex set with $0\in \textnormal{int }F$, and let $\Omega \subset \mathbb{R}^q$ be nonempty and bounded. Then we have the representation
\begin{equation}\label{eq:relationship_MSMG_Minkowski}
\bar{\rho}_F^{\Omega}(x)=\sup_{\omega \in \Omega} \rho_F(x-\omega).
\end{equation} 
\end{proposition}
\begin{proof}
In the case where \(\bar{\rho}_F^{\Omega}(x) = 0\), we clearly have \(\Omega = \{x\}\). Thus \(\sup_{\omega \in \Omega} \rho_F(x - \omega) = 0\), which justifies the equality \eqref{eq:relationship_MSMG_Minkowski} in this case. If \(\bar{\rho}_F^{\Omega}(x) > 0\), pick an arbitrary number \(\epsilon > 0\) and put \(t:= \bar{\rho}_F^{\Omega}(x) - \epsilon>0\). It follows from  the definition of \(\bar{\rho}_F^{\Omega}(x)\) that there exists \(\omega \in \Omega\) such that \(x - \omega \notin tF\), which implies that 
$$
t < \rho_F(x - \omega) \le \sup_{\omega \in \Omega} \rho_F(x - \omega).
$$
Since \(t \to \bar{\rho}_F^{\Omega}(x)\) as \(\epsilon\downarrow 0\), we arrive at \(\bar{\rho}_F^{\Omega}(x) \le \sup_{\omega \in \Omega} \rho_F(x - \omega)\). 

To verify the reverse inequality in \eqref{eq:relationship_MSMG_Minkowski}, let \(\beta\) be such that \(\bar{\rho}_F^{\Omega}(x) \le\beta< \bar{\rho}_F^{\Omega}(x) + \alpha\), where \(\alpha > 0\) is arbitrary. The definition of \(\bar{\rho}_F^{\Omega}(x)\) tells us that \(x - \omega \in \beta F\) for all \(\omega \in \Omega\), and hence $\rho_F(x-\omega)\le\beta$ on $\Omega$. Therefore, we get
$$
\sup_{\omega \in \Omega} \rho_F(x - \omega) \le\beta<\bar{\rho}_F^{\Omega}(x)+\alpha,
$$
which yields the claimed reverse estimate in \eqref{eq:relationship_MSMG_Minkowski} as $\alpha\downarrow 0$ and thus completes the proof.
\end{proof}

The final theorem of this section reveals the key properties of the MSMG function \eqref{eq:MSMG_function} that are fundamental for our practical applications given below.\vspace*{0.05in}

\begin{theorem}\label{theorem:MSMG_properties}
Let $F\subset \mathbb{R}^q$ be a compact and convex set with $0\in \textnormal{int}\,F$, and let $\Omega \subset \mathbb{R}^q$ be nonempty, compact, and convex. Then we have the following assertions:

{\bf(i)} The MSMG function is convex.

{\bf(ii)} The MSMG function is Lipschitz continuous on $\mathbb{R}^q$.

{\bf(iii)} The subdifferential of the MSMG function  at $\bar{x}\in \textnormal{dom}\bar{\rho}_F^{\Omega}$ satisfies the inclusion
\[
\partial \rho_F (\bar{x}-\omega^\prime) \subset \partial \bar{\rho}_F^{\Omega}(\bar{x}),
\]
where $\omega^\prime \in \Bar{\Pi}_F(\bar{x};\Omega):=\{\omega \in \mathbb{R} \ |\  \bar{\rho}_F^{\Omega}(\bar{x})=\rho_F(\bar{x}-\omega) \}$.

{\bf(iv)} The subdifferential of the MSMG function at $\bar{x}\in \textnormal{dom}\, \bar{\rho}_{F}^{\Omega}$ is represented by
\begin{equation}\label{mar-rep}
\partial \bar{\rho}_F^{\Omega}(\bar{x})=\textnormal{co} \left \{ \bigcup_{\omega^\prime \in \Bar{\Pi}_{F}(\bar{x};\Omega)} N(\omega^\prime;\bar{x}-rF)\cap \Upsilon^\prime \right \},
\end{equation}
where $\Upsilon^\prime:=\{v \in \mathbb{R}^q \ | \ \sigma_{-F}(v)=1\}$, $\omega^\prime \in \Bar{\Pi}_F(\bar{x};\Omega)$, and  $r:=\rho_{-F}^{\{\bar{x}\}}(\omega^\prime)$.
\end{theorem}
\begin{proof}
It follows from Proposition~\ref{prop:convexity_of_set_based_minkowski} that $\rho_F$ is convex. Then we deduce from Propositions~\ref{prop:relationship_MSMG_Minkowski} and \ref{prop:convexity_of_sum_funct_and_max_funct} that $\bar{\rho}_F^{\Omega}(\cdot)$ is convex as well, which is claimed in (i). To verify (ii), we have by \parencite[Theorem~6.14]{mordukhovich2023easy} that $\rho_F$ is subadditive. Combining this with Proposition~\ref{prop:relationship_MSMG_Minkowski} ensures the implications
\begin{align*}
& \rho_F(x-\omega) \le \rho_F(x-y) + \rho_F(y-\omega)  \quad \text{for all } \omega \in \Omega \\
& \Longrightarrow \sup_{\omega \in \Omega} \rho_F(x-\omega) \le \rho_F(x-y) + \sup_{\omega \in \Omega}\rho_F(y-\omega)\\
& \Longrightarrow  \sup_{\omega \in \Omega} \rho_F(x-\omega) - \sup_{\omega \in \Omega}\rho_F(y-\omega) \le \rho_F(x-y) \\
& \Longrightarrow  \bar{\rho}_F^{\Omega}(x) -\bar{\rho}_F^{\Omega}(y) \le \rho_F(x-y).
\end{align*}
It follows from  \parencite[Proposition~6.18]{mordukhovich2023easy} that there exists a constant $i\ge 0$ such that \( \rho_F(x) \leq l \|x\| \). Therefore, we have
\[
\bar{\rho}_F^{\Omega}(x) -\bar{\rho}_F^{\Omega}(y) \le \rho_F(x-y)\le l(x-y),
\]
which justifies the Lipschitz continuity of $\bar{\rho}_F^{\Omega}(\cdot)$  on $\mathbb{R}^q$ due to the arbitrary choice of $x$ and $y$. 

To verify assertion (iii), observe that $\Bar{\Pi}_F(\bar{x};\Omega)\ne\emptyset$ as it follows from the classical Weierstrass existence theorem due to the continuity of $\rho_F$ and the imposed compactness of $\Omega$. Picking any subgradient $v \in \partial \rho_F(\bar{x}-\omega^\prime)$, we have 
\[
\langle v, x-\bar{x} \rangle \le \rho_F(x-\omega^\prime)-\rho_F(\bar{x}-\omega^\prime) \le \sup_{\omega \in \Omega} \rho_F(x-\omega)-\sup_{\omega \in \Omega} \rho_F(\bar{x}-\omega)=\bar{\rho}_F^{\Omega}(x)-\bar{\rho}_F^{\Omega}(\bar{x}).
\]
Therefore, $v \in \partial \bar{\rho}_F^{\Omega}(\bar{x})$, which implies that $ \partial \rho_F(\bar{x}-\omega^\prime) \subset \partial \bar{\rho}_F^\Omega(\bar{x})$ as claimed in (iii). 

It remains to prove assertion (iv). Let us first show that 
\begin{equation}\label{minF}
\sup_{\omega \in \Omega} \rho_{-F}^{\{\bar{x}\}}(\omega) = \rho_{-F}^{\{\bar{x}\}}(\omega'),
\end{equation}
where \( \omega'\in\Bar{\Pi}_{F}(\bar{x}; \Omega) \). To this end, apply Proposition~\ref{prop:relationship_between_set_based_and_Minkowski} to $\rho_{-F}^{\{\bar{x}\}}(\omega)$, which gives us 
$\sup_{\omega \in \Omega} \rho_{-F}^{\{\bar{x}\}}(\omega) = \sup_{\omega \in \Omega} \rho_{-F}(\omega - \bar{x}).$
Using Lemma~\ref{lemma:Minkowski_of_F_and_negative_F} from the Appendix, we get
$\rho_{-F}(\omega - \bar{x}) = \rho_F(\bar{x} - \omega)$. It follows from the definition of \( \Bar{\Pi}_F(\bar{x}; \Omega)\) and Proposition~\ref{prop:relationship_MSMG_Minkowski} that any fixed generalized projection vector \( \omega'\in\bar{\Pi}_F(\bar{x}; \Omega) \) satisfies the equality
\[
\rho_F(\bar{x} - \omega') = \sup_{\omega \in \Omega} \rho_F(\bar{x} - \omega),
\]
which can be equivalently written as
\[
\rho_{-F}(\omega' - \bar{x}) = \sup_{\omega \in \Omega} \rho_{-F}(\omega - \bar{x})
\]
and  tells us therefore that
\begin{equation}\label{eq:equation_in_proof_MSMG_subdifferential}
\sup_{\omega \in \Omega} \rho_{-F}^{\{\bar{x}\}}(\omega) = \rho_{-F}^{\{\bar{x}\}}(\omega').
\end{equation}
Employing Proposition~\ref{prop:MSMG_with_negative_F} from the Appendix together with  (\ref{eq:equation_in_proof_MSMG_subdifferential}) leads us to the equalities
\[
\bar{\rho}_F^{\Omega}(\bar{x})=\sup_{\omega \in \Omega} \rho_{-F}^{\{\bar{x}\}}(\omega)=\rho_{-F}^{\{\bar{x}\}}(\omega^\prime).
\]     
Finally, it follows from Theorem~\ref{prop:subdifferential2} that
\[
\partial \rho_{-F}^{\{\bar{x}\}}(\omega^\prime)=N(\omega^\prime;\bar{x}-rF)\cap \Upsilon^\prime.
\]
This allows us to deduce from the subdifferential maximum rule in Theorem~\ref{theorem:max_rule} the fulfillment of representation \eqref{mar-rep} and thus complete the proof of this theorem.
\end{proof}

\section{Well-Posedness of the Generalized SFT Model}\label{sec:existence-uniqueness}

In this section, we verify the {\em well-posedness} of our main generalized SFT model and hence of its specification, meaning by well-posedness the {\em existence} and {\em uniqueness} of optimal solutions. Given two points  \( x, y \in \mathbb{R}^q \), denote the straight line passing through them by
\begin{equation} \label{eq:connecting-line}
L(x, y) := \{ \lambda x + (1 - \lambda)y \mid \lambda \in \mathbb{R} \}.
\end{equation}

The first theorem establishes the {\em existence} of optimal solutions to the generalized SFT problem \eqref{eq:SFT} under a simple assumption.\vspace*{0.05in}

\begin{theorem}\label{theorem:SFT-existance}
In the setting of the generalized SFT problem~\eqref{eq:SFT}, an optimal solution exists if at least one of the sets $\Omega_0$ and $\Omega_i$ as $i \in I_k$, $k = 1, \dots, n$, is bounded.
\end{theorem}
\begin{proof}
For any $\lambda\ge 0$, define the {\em sublevel set} associated with the function $S(x)$ from \eqref{eq:SFT-definition} relative to the constraint set $\Omega_0$ in \eqref{eq:SFT} by
\begin{equation*}
\mathcal{L}_{\lambda}(S) :=\big\{x \in \Omega_0\;\big|\;S(x) \le \lambda\big\}.
\end{equation*}
This implies that \( \sum_{i \in I_k} \rho_{F_i}^{\Omega_i}(x) \le \lambda \) as $k = 1, \dots, n$ whenever $x \in \Omega_0$. Therefore, we have $\mathcal{L}_\lambda(S) \subset A$ with the set $A$ given by
\[
A := \big\{x \in \Omega_0\;\big|\;\rho_{F_i}^{\Omega_i}(x) \le \lambda \ \text{for all} \ i \in I_k, \ k = 1, \dots, n \big\},
\]
Since the sets  $\Omega_i$ are closed, it follows from $\rho_{F_i}^{\Omega_i}(x) \le \lambda$ that $x \in \Omega_i + \lambda F_i$. Thus
\[
A = \Omega_0 \cap \left( \bigcap_{k=1}^{n} \left( \bigcap_{i \in I_k} (\Omega_i + \lambda F_i) \right) \right).
\]
This tells us that the boundedness of one of the sets $\Omega_i$ for some $i \in I_k$ with $k = 1,\dots, n$ yields the boundedness of the set $A$, and hence this property of the sublevel set $\mathcal{L}_\lambda(S)$. The existence of optimal solutions to \eqref{eq:SFT} follows now from 
\parencite[Corollary 7.10]{mordukhovich2023easy} due to the classical Weierstrass existence theorem.
\end{proof}

Before deriving the uniqueness result for \eqref{eq:SFT}, we present the following lemma of its independent interest, which is used in the proof of the uniqueness theorem below.\vspace*{0.05in{}}

\begin{lemma}\label{prop:strictly_convex_fermat_Torricelli} In addition to the standing assumptions in the formulation of the generalized SFT problem \eqref{eq:SFT}, suppose that the sets $F_i$ and $\Omega_i$, $i=1,\dots, m$, are strictly convex. If furthermore for any line $L(x, y)$ from \eqref{eq:connecting-line} with $x, y \in \Omega_0$ and $x \neq y$, there exists $j \in \{1, \dots, m\}$ such that $L(x, y) \cap \Omega_j = \emptyset$, then the function $T(x)$ defined in \eqref{eq:Generalized-Fermat-equation} is strictly convex on the set $\Omega_0$.
\end{lemma}
\begin{proof}
It follows from Propositions~\ref{prop:convexity_of_sum_funct_and_max_funct}, \ref{prop:convexity_of_set_based_minkowski} that \( T(x) \) is convex.
Suppose on the contrary that \( T(x) \) is not strictly convex. Then there exist \( x, y \in \Omega_0 \) with \( x \neq y \) and \( \lambda \in (0,1) \) such that
\[
T(\lambda x + (1 - \lambda)y) = \lambda T(x) + (1 - \lambda) T(y).
\]
Since each \( \rho_{F_i}^{\Omega_i} \) is convex, we get by definition that
\[
\rho_{F_i}^{\Omega_i}(\lambda x + (1 - \lambda)y)\le\lambda \rho_{F_i}^{\Omega_i}(x) + (1 - \lambda) \rho_{F_i}^{\Omega_i}(y)\;\text{ whenever } \  i = 1,\dots, n.
\]
As assumed, there exists \( j \) with \( L(x, y) \cap \Omega_j = \emptyset \).  Pick \( \omega \in \Pi_{F_j}(x; \Omega_j) \), \( v \in \Pi_{F_j}(y; \Omega_j) \)  and then write \( \rho_{F_j}^{\Omega_j}(x) = \rho_{F_j}(x - \omega) \) and \( \rho_{F_j}^{\Omega_j}(y) = \rho_{F_j}(y - v) \).  This gives us the relationships
\begin{equation*}
\begin{split}
\lambda \rho_{F_j}(x - \omega) + (1 - \lambda) \rho_{F_j}(y - v) 
&= \lambda \rho_{F_j}^{\Omega_j}(x) + (1 - \lambda) \rho_{F_j}^{\Omega_j}(y) \\
&= \rho_{F_j}^{\Omega_j}(\lambda x + (1 - \lambda) y) \\
&\le \rho_{F_j}(\lambda x + (1 - \lambda) y - (\lambda \omega + (1 - \lambda) v)) \\
&\le \lambda \rho_{F_j}(x - \omega) + (1 - \lambda) \rho_{F_j}(y - v),
\end{split}
\end{equation*}
which imply therefore that $ \lambda \omega+(1-\lambda)v\in \Pi_{F_j}(\lambda x+(1-\lambda)y;\Omega_j)$. We know by Proposition~\ref{prop:projection&border} that $\omega \in \textnormal{bd}(\Omega_j)$, $v\in \textnormal{bd}(\Omega_j)$, and $\lambda \omega+(1-\lambda)v\in \textnormal{bd}(\Omega_j)$. Since the set $\Omega_j$ is strictly convex, it follows that $\omega=v$. Furthermore, the function \( \rho_F \) is positively homogeneous, and hence
\[
\rho_{F_j}( \lambda x+(1-\lambda)y-\omega)= \lambda \rho_{F_j}(x-\omega)+(1-\lambda)\rho_{F_j}(y-\omega)=\rho_{F_j}(\lambda(x-\omega))+\rho_{F_j}((1-\lambda)(y-\omega)).
\]
By $x,y \notin \Omega_j$ we get $x-\omega, y-\omega \ne 0$.  Moreover,
Proposition~\ref{prop:equality_minkowski_strictly_convex} tells us that $\alpha>0$ with $\lambda(x-\omega)=\alpha(1-\lambda)(y-\omega)$. This brings us to the representation
\[
x-\omega=\beta(y-\omega)\ \ \textnormal{for} \ \ \beta=\frac{\alpha (1-\lambda)}{\lambda}\ne 1,
\]
which indicates  that $\omega=\frac{1}{1-\beta}x-\frac{\beta}{1-\beta}y \in L(x,y)$. Since $\omega \in \Omega_j$, we get  $L(x,y) \ \cap \ \Omega_j \neq \emptyset$, a  contradiction that  comoletes the proof of the theorem. 
\end{proof}

Finally, we are ready to establish the {\em uniqueness} of optimal solutions to the generalized SFT problem \eqref{eq:SFT} under appropriate assumptions. 
\vspace*{0.05in}

\begin{theorem}\label{theorem:SFT-uniqness} In addition to the standing assumptions on the data of \eqref{eq:SFT}, suppose that:

{\bf(i)} The sets $F_i$ and $\Omega_i$ are strictly convex for $i \in I_k$ with $k=1,\dots,n$.

{\bf(ii)} For each index set \( I_k \) with \( k \in \{1,\dots, n\} \), there exists a set \( \Omega_j \) among \( \Omega_i \) with \( i \in I_k \) such that \( L(x, y) \cap \Omega_j = \emptyset \) whenever \( x, y \in \Omega_0 \) with $x\ne y$. 

{\bf(iii)} At least one of the sets $\Omega_0$ and  \( \Omega_i \), $i\in I_k$, \( k = 1, \dots, n \), is bounded.\\[0.5ex]
Then problem \eqref{eq:SFT} admits a unique optimal solution.

\end{theorem}
\begin{proof}
The existence of solutions to \eqref{eq:SFT} is proved in Theorem~\ref{theorem:SFT-existance}. To verify the uniqueness, we get from Lemma~\ref{prop:strictly_convex_fermat_Torricelli} that the sum $\sum_{i \in I_k} \rho_{F_i}^{\Omega_i}(x)$ is strictly convex for each $k=1,\dots,n$. Furthermore,  Proposition~\ref{lemma:max_strictly_convex} tells us that the  function  \( S(x) \) is strictly convex as well, which therefore ensures the claimed solution uniqueness.
\end{proof}

\section{Weighted and  Extended Formulations of the Generalized SFT Problem}\label{weight}

Motivated by prior works that studied the weighted generalized Fermat–Torricelli and Sylvester problems \parencite{mordukhovich2012applications,nickel2003approach, plastria2009asymmetric,PLASTRIA201698,PELEGRIN1985327}, we conclude this section by presenting and discussing the \emph{weighted formulation} of the generalized SFT problem given, in the setting of (\ref{eq:SFT}), by
\begin{equation*}
\min\ \tilde{S}(x) \quad \text{subject to} \quad x \in \Omega_0,
\end{equation*}
where the objective function is defined, with the weights \( \omega_i > 0 \) for \( i \in I_k \) and \( k = 1, \dots, n \), as
\begin{equation}
\tilde{S}(x) := \max \Big\{ \sum_{i \in I_k} \omega_i \rho_{F_i}^{\Omega_i}(x) \ \Big| \ k = 1, \dots, n\Big\}.
\end{equation}
It is worth noting that the original formulation of the generalized SFT problem in (\ref{eq:SFT}) inherently includes its weighted counterpart. Indeed, it follows from Proposition~\ref{prop:positive_scaled_set_based_minkowski} that
\[
\omega_i \rho_{F_i}^{\Omega_i}(x) = \rho_{\frac{F_i}{\omega_i}}^{\Omega_i}(x),
\]
which allows us to express the weighted version in the same form as the original problem \eqref{eq:SFT}
where the (modified) cost function is given by
\begin{equation*}
S(x) := \max\Big\{ \sum_{i \in I_k} \rho_{\tilde{F}_i}^{\Omega_i}(x) \ \Big| \ k = 1, \dots, n \Big\}\;\mbox{ with }\;\tilde{F}_i := F_i / \omega_i.
\end{equation*}
{\color{black} This facilitates the investigation of well-posedness for the weighted generalized SFT problem using Section~\ref{sec:existence-uniqueness}.}

Furthermore, we can explore the following {\em extended version} of the generalized SFT problem that incorporates both the set-based Minkowski gauge and the MSMG functions. Let $\bar{m} \in \mathbb{N}$ and the index sets \( J_k \neq \emptyset \) for \( k = 1, \dots, n \) form a partition of \( J := \{1,\dots, \bar{m}\} \) such that 
\begin{equation*}
\bigcup_{k=1}^{n} J_{k} = J\;\mbox{ and }\;J_{k} \cap J_{l} = \emptyset \;\mbox{ whenever }\; k \neq l\;\mbox{ with }\, k, l \in \{1,\dots, n\}.
\end{equation*}
For $j \in J_k$ with $k=1,\dots n$, let \( \bar{F}_j \subset \mathbb{R}^q \) be a compact and convex set with \( 0 \in \operatorname{int}\,\bar{F}_j \), and let \( \Theta_j \) be nonempty, compact, and convex subsets of \( \mathbb{R}^q \). The extension of \eqref{eq:SFT} is formulated by
\begin{equation}\label{eq:SFT-extension}
\min\;\widehat{S}(x)\;\mbox{ subjet to  }\; x\in \Omega_0 , 
\end{equation}
where the  extended cost function is defined by
\begin{equation}\label{eq:SFT-extension-definition}
\widehat{S}(x):=\max\Big\{\sum_{i \in I_k} \rho_{F_i}^{\Omega_i}(x) +\sum_{j \in J_{k}}\bar{\rho}_{\bar{F}_j}^{\Theta_j}(x)\  \Big| \ k=1,...,n\Big\}.
\end{equation}
Applications of the generalized SFT problem in the extended form (\ref{eq:SFT-extension}) are demonstrated in a real-world scenario in Section~\ref{sec:application}.

In the case where $k=1$, $|I_1|=m$, and $|J_1|=\bar{m}$ in problem (\ref{eq:SFT-extension}), the extended version of the generalized Fermat-Torricelli problem is represented by
\begin{equation}\label{eq:eq:Extended-Generalized-Fermat-Model}
\min \Hat{T}(x)\;\mbox{ subject to }\; x\in \Omega_0
\end{equation}
with the cost function given in the summation form
\begin{equation}\label{eq:Extended-Generalized-Fermat-equation}
\Hat{T}(x):= \sum_{i=1}^{m} \rho_{F_i}^{\Omega_i}(x)+\sum_{j=1}^{\bar{m}} \Bar{\rho}_{\Bar{F}_j}^{\Theta_j}(x).
\end{equation}
Applications of \eqref{eq:eq:Extended-Generalized-Fermat-Model}, \eqref{eq:Extended-Generalized-Fermat-equation} to disaster relief operations using UAVs are discussed in Section~\ref{sec:Multiple_Example_MSMG}.

Finally in this section, we formulate an extended version of the generalized Sylvester problem  given by  
\begin{equation}\label{eq:extention-sylvester}
\min \Hat{Y}(x)\;\mbox{ subject to }\;x \in \Omega_0,
\end{equation}
where  both components of the cost functions are represented in the maximun function form
\begin{equation}\label{eq:extension-sylvester-defenition}
\Hat{Y}(x) =\big\{ \mathcal{Y}_1(x), \mathcal{Y}_2(x)\big\}\;\mbox{ with }\;\mathcal{Y}_1(x): =\max_{i=1,\dots,m} \rho_{F_i}^{\Omega_i}(x), 
\quad 
\mathcal{Y}_2(x): = \max_{j=1,\dots,\bar{m}} \Bar{\rho}_{F_j}^{\Theta_j}(x).
\end{equation}
Both versions of the generalized Sylvester problem  proposed in (\ref{eq:generalized-Syvester-Model}) and (\ref{eq:extention-sylvester}) constitute novel models whose applications to UAV-based disaster relief operations are discussed in 
Sections~\ref{sec:Multiple_UAV_example} and \ref{sec:Multiple_Example_MSMG}. Observe that the convexity of the functions $\Hat{S}(x)$, $\Hat{T}(x)$, and $\Hat{Y}(x)$, defined in (\ref{eq:SFT-extension-definition}), (\ref{eq:Extended-Generalized-Fermat-equation}), and (\ref{eq:extension-sylvester-defenition}), respectively, follows directly from Propositions~\ref{prop:convexity_of_set_based_minkowski}, \ref{prop:convexity_of_sum_funct_and_max_funct} and Theorem~\ref{theorem:MSMG_properties}. Furthermore, their subgradient evaluations can be obtained by using Theorems~\ref{prop:subdifferential_1}, \ref{prop:subdifferential2} and Proposition~\ref{prop:subdifferential_estimation} combined with Theorems~\ref{theorem:MSMG_properties}, \ref{theorem:max_rule}, and \ref{theorem:sum_rule_subdifferential}.

\section{Applications of the New Location Science Models to Disaster Relief Operations} \label{sec:application}

After an earthquake, some people may be trapped under debris but still alive. If they have access to their cell phones, one of the most effective ways they can help themselves is by sending their location to family members or posting it on social media {(see, e.g., \parencite{toraman2023tweets} for a real case). This can greatly increase their chances of being found and rescued. However, earthquakes often damage or destroy communication systems such as cell towers and internet cables. As a result, even if someone has a working phone, it may not be possible to send messages or connect to the Internet.

In such cases, UAVs may play a crucial role. By flying over the damaged area, UAVs can provide wireless connections or DTNs (delay-tolerant networks) while allowing trapped individuals to send their location information (see Fig.~\ref{fig:earth-quake-senario}). This support is able make rescue operations faster and more effective, and thus help saving more lives. Despite their advantages, UAVs also present several challenges; in particular, they are limited by short battery life and are vulnerable to environmental conditions, especially high winds. Moreover, rescue teams often deploy a heterogeneous fleet consisting of UAVs with varying speeds, ranges, and operational capabilities adding complexity to coordination and mission planning.

In this section, we examine several scenarios in which a truck serves as a mobile station for deploying and recharging UAVs. The primary objective is to determine the optimal location of the truck (see Fig. \ref{fig:earth-quake-senario}) to minimize the \emph{transition time}, i.e., the period when the UAV is traveling between the service area and the mobile station, resulting in a temporary loss of network service. Such models can be formulated via single-facility location problems, which have been extensively studied in the literature by using the $\ell_1$ and $\ell_2$ norms. However, to model a more realistic scenario, we incorporate the effect of wind that prevents us from employing these common norms directly. Using the $\ell_2$ norm, for example, would ignore the wind's influence, leading us to a simplification that can result in significant time loss. 

To minimize transition time, we use the {\em set-based Minkowski gauge function} (\ref{eq:set-based-minkowski-definition}) that captures the effects of wind, which are inherent in real-world conditions. In the first scenario, a rescue team utilizes a single UAV, and we show how the generalized Fermat-Torricelli problem can be applied to minimize transition time. In the second scenario, the rescue team operates a heterogeneous fleet of UAVs, and we illustrate how the generalized Sylvester problem can effectively address the complexities that arise. Finally, we consider the generalized SFT problem in a scenario where the heterogeneous fleet of UAVs is first deployed to the affected area and then returns to the truck for recharging and data transfer. This setup highlights the operational challenges and the importance of strategic truck placement in dynamic rescue missions.

\begin{figure}
\centering
\includegraphics[width=0.99\linewidth]{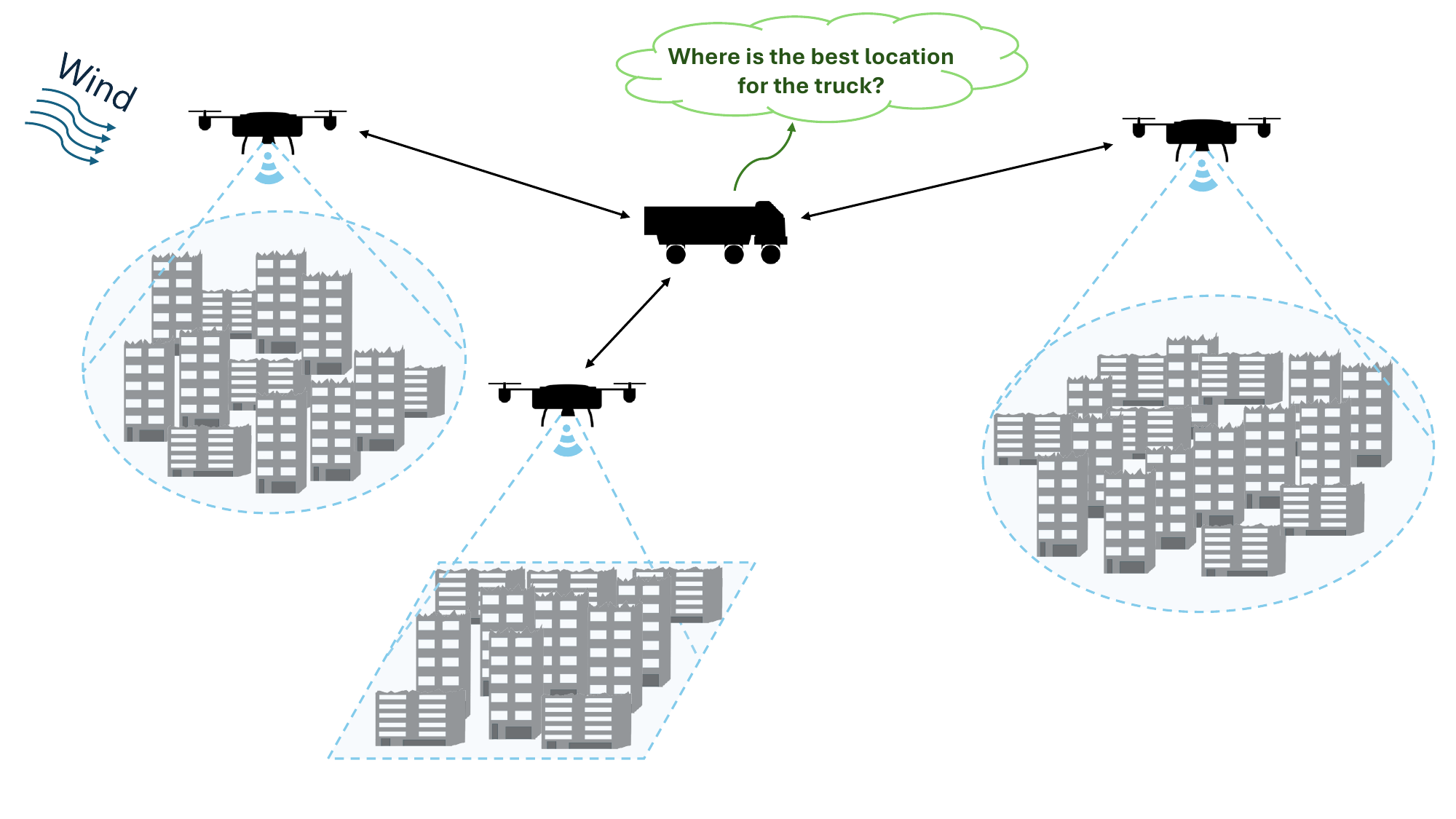}
\caption{Providing a wireless network or DTNs after an earthquake is one of the important applications of UAVs in disaster relief}
\label{fig:earth-quake-senario}
\end{figure}

\subsection{Applications of the Generalized Fermat-Torricelli Problem}\label{sec:single_UAV_example}

Owing to constraints such as limited budget, it is often the case that a rescue team operates with only a single UAV. In this section, we examine an operation in which the rescue team employs a UAV to provide DTN support for three predetermined areas. To begin with, the UAV is dispatched from the mobile station (e.g., a truck) and covers the first area. After collecting data, the UAV returns to the station for recharging and transferring the collected data to the rescue team. This pattern is then repeated for the remaining two areas. To determine the optimal location of the truck in order to minimize transition time under windy conditions, we apply the above generalized Fermat-Torricelli model. We present a step-by-step solution to illustrate how our model is implemented in this context.

To formulate the aforementioned challenge (corresponding to Case~1 in Table \ref{tab:Fermat-Toriceili-results}) as a generalized Fermat-Torricelli problem (\ref{eq:eq:Generalized-Fermat-Model}), we begin by specifying the reference and dynamic sets. The reference sets are given as follows:
\begin{align*}
\Omega_0 &= \mathbb{R}^2, \\
\Omega_1 &= \left\{(x, y) \in \mathbb{R}^2 \mid \max\{|x - 30|, |y - 350|\} \le 15 \right\}, \\
\Omega_2 &= \left\{(x, y) \in \mathbb{R}^2 \mid \max\{|x - 210|, |y - 10|\} \le 15 \right\}, \\
\Omega_3 &= \left\{(x, y) \in \mathbb{R}^2 \mid \max \{|x - 550| , |y - 200| \} \le 15 \right\}.
\end{align*}
The dynamic sets in the absence of wind, corresponding to the UAV’s nominal speed \( r \) (m/s), are represented by the ball of the plane:
\begin{equation*}
F_0 = \left\{(f_1, f_2) \in \mathbb{R}^2 \mid f_1^2 + f_2^2 \le r^2 \right\}\ \  \textnormal{for}\  \ i=1,2,3.
\end{equation*} 
Taking wind into account, with wind vector $ \mathcal{S}=(s_1,s_2)\in \mathbb{R}^2$, the effective dynamic set--representing the UAV's velocity under wind conditions--is given by
\begin{equation}\label{eq:dynamic_under_wind_single_UAv_example}
F_i=F = \left\{(f_1, f_2) \in \mathbb{R}^2 \mid (f_1 - s_1)^2 + (f_2 - s_2)^2 \le r^2 \right\} \ \ \textnormal{for}\ \ i=1,2,3.
\end{equation}
\begin{figure}
\centering
\subfigure(a){\includegraphics[width=0.424\textwidth]
{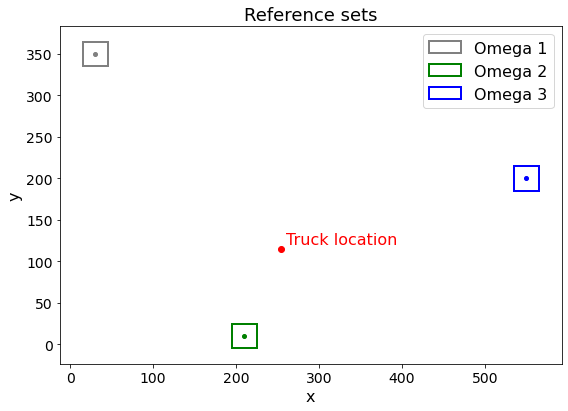}} 
\hspace{0.5 cm}
\subfigure(b){\includegraphics[width=0.4\textwidth]{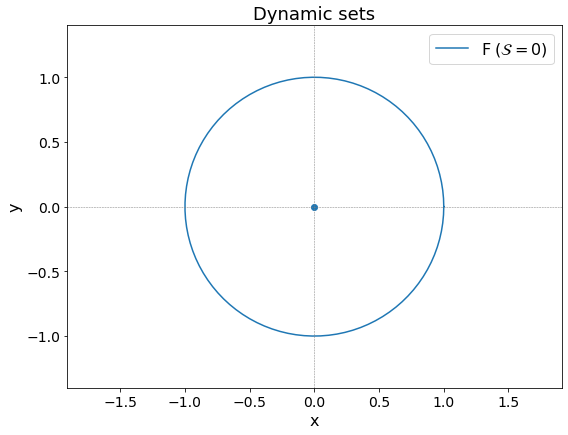}} 
\subfigure(c){\includegraphics[width=0.424\textwidth]
{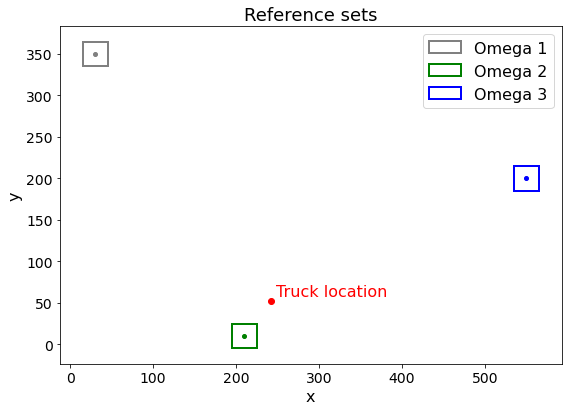}} 
\hspace{0.5 cm}
\subfigure(d){\includegraphics[width=0.4\textwidth]{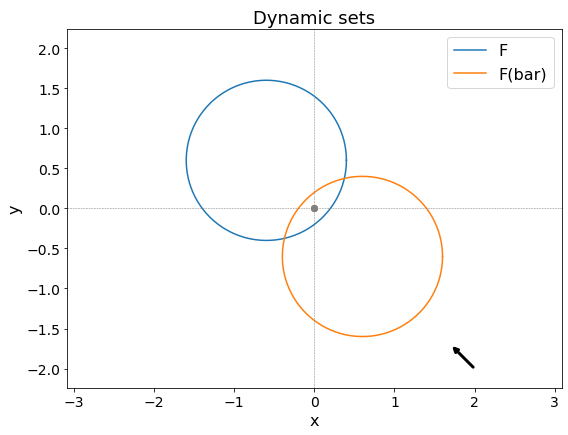}} 
\caption{Case 1 in Table~\ref{tab:Fermat-Toriceili-results}:  
(a) Approximated optimal location of the truck when the wind effect is neglected.
(b) Dynamic set with $\|\mathcal{S}\|=0$.  
(c) Approximated optimal location of the truck under windy conditions.  
(d) Dynamic sets of the UAV under wind conditions with $\mathcal{S}=(-0.6,0.6)$.  
The black arrow indicates the wind direction}
\label{fig:example_fermat}
\end{figure}
Dispatching the UAV from the mobile station involves movement in the opposite direction, so the corresponding dynamic set is given by  $-F_i$ for $i=1,2,3$, as stated in Proposition~\ref{prop:SMG_with_negative_F} from the Appendix. We assume that the wind vector satisfies the condition $\left\| \mathcal{S}\right\|_2 < r$, ensuring that the UAV can still operate effectively under the given wind condition. Therefore, the generalized Fermat-Torricelli problem is represented as (assuming for simplicity that \( \bar{F} = -F \)):
\begin{equation}\label{eq:generalized-fermat-torocelli-case-1}
\min_{x \in \Omega_0} T(x) \ \ \textnormal{with}\ \ T(x)=\sum_{i=1}^{3}\rho_{F_i}^{\Omega_i}(x)+\sum_{i=1}^{3}\rho_{-F_i}^{\Omega_i}(x)=\sum_{i=1}^{3}\rho_{F}^{\Omega_i}(x)+\sum_{i=1}^{3}\rho_{\bar{F}}^{\Omega_i}(x),
\end{equation}
where $x=(x_1,x_2) \in \mathbb{R}^2$.
To solve \eqref{eq:generalized-fermat-torocelli-case-1}, we split the solution procedure into several steps.\\[1ex]
\textbf{Step~1: Computing the Minkowski gauge function $\rho_F(x).$}\\
The classical  Minkowski gauge $\rho_F(x)$ in (\ref{eq:minkowski_definition}) is the optimal value of the following problem:
\begin{equation} \label{eq:example_1_rho}
\begin{split}
&\min \ t\;\mbox{ subject to}\\
& (f_1 - s_1)^2 + (f_2 - s_2)^2 = r^2,\\
&x_1 = t f_1,\;x_2 = t f_2,\;t \geq 0.
\end{split}
\end{equation}
According to Proposition~\ref{prop:f_on_border}, the first constraint in (\ref{eq:example_1_rho}) enforces that the velocity vector \( (f_1, f_2) \) lies on the boundary of the dynamic set \( F \). To derive a closed-form expression for \( \rho_F(x) \), we substitute \( f_1 = x_1 / t \) and \( f_2 = x_2 / t \) into the first constraint in (\ref{eq:example_1_rho}) and get
\begin{equation}\label{eq:quadratic_equation_to_find_minkowski}
\begin{split}
&\ \displaystyle\Big(\frac{x_1}{t}-s_1\Big)^2+\Big(\frac{x_2}{t}-s_2\Big)^2=r^2\\
\Longrightarrow &\ (x_1-ts_1)^2+(x_2-ts_2)^2=r^2t^2\\
\Longrightarrow & \ t^2(s_1^2+s_2^2-r^2)-2(s_1x_1+s_2x_2)t+(x_1^2+x_2^2)=0.
\end{split}
\end{equation}
Solving the quadratic equation from (\ref{eq:quadratic_equation_to_find_minkowski}) gives us
\begin{equation}\label{eq:obtained_t_to_minkowski}
\rho_F(x)=t=\displaystyle\frac{(s_1x_1+s_2x_2)-\sqrt{x_1^2(r^2-s_2^2)+x_2^2(r^2-s_1^2)+2s_1s_2x_1x_2}}{s_1^2+s_2^2-r^2}.
\end{equation}
The value of \( t \) obtained from (\ref{eq:obtained_t_to_minkowski}) is the optimal solution to the problem in (\ref{eq:example_1_rho}), and it is obviously nonnegative.\\[1ex]
\textbf{Step~2: Computing subgradients of the functions $\rho_F(x)$ and $\rho_F^{\Omega}(x)$.}\\
Here we compute subgradients of the functions 
\( \rho_F(x) \) and \( \rho_F^{\Omega}(x) \), which are used below in the subgradient algorithm  (Algorithm~\ref{alg:subgradient}) to solve problem  (\ref{eq:generalized-fermat-torocelli-case-1}). Define $\alpha:=\frac{1}{s_1^2+s_2^2-r^2}$ and
\[\beta(x_1,x_2):=x_1^2(r^2-s_2^2)+x_2^2(r^2-s_1^2)+2s_1s_2x_1x_2.\] Then for any $x \ne 0$, the gradient of the function \( \rho_F(x) \) in (\ref{eq:obtained_t_to_minkowski}) is calculated by
\begin{equation}\label{eq:example_1_rho_gradient}
\nabla \rho _F(x)=\alpha(s_1-\frac{x_1(r^2-s_2^2)+s_1s_2x_2}{\sqrt{\beta(x_1,x_2)}},\;s_2-\frac{x_2(r^2-s_1^2)+s_1s_2x_1}{\sqrt{\beta(x_1,x_2)}}).
\end{equation}
Therefore, the subdifferential of the function \( \rho_F^{\Omega_i}\) are calculated by Theorems~\ref{prop:subdifferential_1}  and Proposition~\ref{prop:subdifferential_estimation} as follows:
\begin{equation} \label{eq:example1_subdifferential}
\partial \rho_F^{\Omega_i}(x) = 
\begin{cases}
N(x; \Omega_i) \cap \mathcal{V} & \text{if } x \in \Omega_i, \\
\{\nabla \rho_F(x - \omega_i)\} & \text{if } x \notin \Omega_i,
\end{cases}
\quad i = 1, 2, 3,
\end{equation}
where \(\omega_i \in \Pi_F(x; \Omega_i)\) and $x-\omega_i \ne 0$ if  \( x \notin \Omega_i \).\\[1ex]
\begin{table}
\caption{Experimental Results for Generalized Fermat-Torricelli Cases}
\vspace{0.25cm}
\begin{tabular}{ccc|cc|cc|cc}
\hline
\multicolumn{1}{c}{Case} &
\multicolumn{1}{c}{Parameters} &
\multicolumn{1}{c|}{Wind Vector}&
\multicolumn{2}{c|}{\cellcolor{gray!20} Wind Neglected } &
\multicolumn{2}{c|}{ \cellcolor{gray!20} Wind Included} &
\multicolumn{1}{c}{Absolute  }&
\multicolumn{1}{c}{Relative}\\ 
\multicolumn{1}{c}{No.} &
\multicolumn{1}{c}{(Table \ref{tab:experiments-settings})} &
\multicolumn{1}{c|}{$(\mathcal{S}=(s_1,s_2))$}&
\multicolumn{1}{c}{$x^\ast$} &
\multicolumn{1}{c|}{\(Z_N(s)\)} &
\multicolumn{1}{c}{$x^\ast$} &
\multicolumn{1}{c|}{\(Z_I(s) \)} &
\multicolumn{1}{c}{Imp(s)} &
\multicolumn{1}{c}{Imp($\%$)}
\\ \hline
1&Info 1&$(-0.6,0.6)$&(252,117)&4039 &(242,52)&3959&80&1.98\\
2&Info 1&$(0.6,0.6)$&(252,117)&3626 &(225,171)&3575&53&1.46\\
3&Info 1&$(0.7,0.3)$&(252,117)&2746 &(228,108)&2727&19&0.00\\
4&Info 1&$(-0.1,0.9)$&(252,117)&5214 &(350,162)&5037&177&3.39\\
5&Info 1&$(0.4,-0.1)$&$(252,117)$&1605 &$(247,102)$&1604&1&0.00\\
6&Info 2&$(0.8,-0.3)$&(604,265)& 8305 &(581,101)&8148&157&1.89\\
7&Info 2&$(0,-0.2)$&(604,265)&2728 &(605,270)&2728&0&0.00\\
8&Info 3&$(0.4,-0.8)$&$(166,48)$&3103&(175,35)&3061&42&1.35\\
9&Info 3&(0.3,0)&(166,48)&896 &(164,44)&896&0&0.00\\
10&Info 3&$(0.4,0.8)$&(166,48)&2742 &(176,84)&2597&145&5.28\\
\hline
\end{tabular}
\vspace{0.1cm}
\label{tab:Fermat-Toriceili-results}
\vspace{0.3cm}
{\scriptsize $Z_N$ and $Z_I$ denote the objective function values (transition time in seconds) for the cases ignoring and considering wind, respectively. Absolute Imp ($Z_N-Z_I$) and Relative Imp ($(Z_N-Z_I)/Z_N$) denote the absolute and relative improvements.}
\end{table}
\textbf{Step~3: Computing the optimal solution by subgradient algorithm.}\\
To determine the optimal location of the mobile station, we apply the subgradient algorithm (see Algorithm~\ref{alg:subgradient}  below) allowing us to find an optimal solution to the generalized Fermat--Torricelli problem under consideration in (\ref{eq:generalized-fermat-torocelli-case-1}). 

\begin{algorithm}
\caption{Subgradient Method}
\label{alg:subgradient}
\begin{algorithmic}[1]
\State \textbf{Input:} Initial point $x^1 \in \Omega_0$, step size sequence $\{\alpha^k > 0\}$, maximum iterations $K\in \mathbb{N}$.
\State \textbf{Initialize:} Set $k = 1$ ($k\in \mathbb{N}$).
\While{$k < K$}
\State Choose a subgradient $v^k \in \partial f(x^k)$.
\State Update: $x^{k+1} = \Pi(x^k - \alpha^k v^k;\Omega_0)$ 
\Comment{$\Pi(.;\Omega_0)$ denotes Euclidean projection onto $\Omega_0$}
\State $k \gets k + 1$
\EndWhile
\State \textbf{Output:} $x^K$
\end{algorithmic}
\vspace{0.1em}
{\footnotesize\textit{Note.}  $\{\alpha^k\}_{k=1}^{\infty}$ should satisfies (1) $\alpha^k \to 0$ and (2) $\sum_{k=1}^{\infty} \alpha^k=\infty$}
\end{algorithm}
According to the structure of the subgradient algorithm, we need to compute a subgradient of \( T(x) \). To do so, we calculate the subdifferential of \( T(x) \) by applying the subdifferential sum rule from Theorem~\ref{theorem:sum_rule_subdifferential}, which gives us the representation
\begin{equation}\label{eq:example_1_subdifferential_T}
\partial T(x) = \sum_{i=1}^{3} \partial \rho_F^{\Omega_i}(x)+\sum_{i=1}^{3}\partial\rho_{\bar{F}}^{\Omega_i}(x).
\end{equation}
Therefore, at each iteration \(k\) of the subgradient algorithm, a subgradient \(v^k \in \partial T(x^k)\) can be calculated as follows:
\begin{equation}\label{eq:example_fermat-torriceili_total_subgradient_calculations}
v^k = \sum_{i=1}^3 v_i^k+\sum_{i=1}^{3}\bar{v}_i^{k},
\end{equation}
where \(v_i^k \in \partial \rho_F^{\Omega_i}(x^k)\) and \(\bar{v}_i^k \in \partial \rho_{\bar{F}}^{\Omega_i}(x^k)\) are taken from (\ref{eq:example1_subdifferential}) as 
\begin{equation}\label{eq:example_fermat_subgradient_two_cases}
v_i^k = 
\begin{cases}
0 & \text{if } x^k \in \Omega_i, \\
\nabla \rho_F(x^k - \omega_i^k) & \text{if } x^k \notin \Omega_i,
\end{cases}
\quad i = 1, 2, 3,
\end{equation}
with \(\omega_i^k \in \Pi_F(x^k; \Omega_i)\). Similarly we get
\begin{equation}
\bar{v}_i^k = 
\begin{cases}
0 & \text{if } x^k \in \Omega_i, \\
\nabla \rho_{\Bar{F}}(x^k - \bar{\omega}_i^k) & \text{if } x^k \notin \Omega_i,
\end{cases}
\quad i = 1, 2, 3,
\end{equation}
where $\bar{\omega}_i^k\in \Pi_{\Bar{F}}(x^k;\Omega_i)$. Note that by solving the problem in (\ref{eq:example_1_rho}) with \( F \) replaced by \( \Bar{F} \), we can compute both \( \rho_{\Bar{F}}(x) \) and its gradient \( \nabla \rho_{\Bar{F}}(x) \).

To illustrate the computation procedure, we focus for definiteness on Case~1 in Table~\ref{tab:Fermat-Toriceili-results} and evaluate 
$v^k \in \partial T(x^k)$ for $k=1$, where the dynamic sets are specified by Info~1 in 
Table~\ref{tab:experiments-settings} and the wind vector is given by $\mathcal{S}=(-0.6,0.6)$.
Therefore, the dynamic sets for this UAV, corresponding to back-and-forth movements, are given by
\begin{equation}\label{eq:example_fermat_back_and_forth_dynamic_sets}
\begin{split}
F=\big\{(f_1,f_2)\in \mathbb{R}^2\big|\;(f_1+0.6)^2+(f_2-0.6)^2 \le 1\big\},\\
\bar{F}=\big\{(f_1,f_2)\in \mathbb{R}^2\;\big|\;(f_1-0.6)^2+(f_2+0.6)^2 \le 1\big\}.
\end{split}
\end{equation}
Starting with $x^1=(100,100)\notin\O_1$, compute $v_1^1$ by (\ref{eq:example_fermat_subgradient_two_cases})  With $\omega_1^1=(45,335)$. This gives us
\[
x^1 - \omega_1^1 = (100-45,\,100-335) = (55,-235).
\]
Then it follows from (\ref{eq:example_1_rho_gradient}) that
\[
\alpha = -3.57, \quad \beta(55,-235) = 46586, \quad \text{and} \quad \nabla \rho_F(55,-235) =
\]
\[
(-3.57) \Bigg(
-0.6 - \frac{55(0.64)+(-235)(-0.36)}{\sqrt{46586}}, \; 
0.6 - \frac{(-235)(0.64)+55(-0.36)}{\sqrt{46586}} \Bigg)
= (4.12, -4.95).
\]
Using the same procedure, we get
\[
\begin{aligned}
v_2^1 &= (-0.46, 0.29), & v_3^1 &= (-0.66, -0.98),\\
\bar{v}_1^1 &= (-0.01, -0.62), & \bar{v}_2^1 &= (-4.74, 4.58), & \bar{v}_3^1 &= (-4.91, 3.11).
\end{aligned}
\]
Finally, $v^1$ can be computed by using (\ref{eq:example_fermat-torriceili_total_subgradient_calculations}), which brings us to
\[
v^1 = (-6.66, 1.43).
\]
In the same way, we compute $v^k$ at each iteration. Then Algorithm~\ref{alg:subgradient} with the standard choice of $\alpha^k=\tfrac{1}{k}$  allows us to solve the generalized Fermat-Torricelli problem (\ref{eq:generalized-fermat-torocelli-case-1}).

To proceed with Case~1 in Table~\ref{tab:Fermat-Toriceili-results}, we first simplify the problem by neglecting the wind effect and considering the dynamic sets for the back-and-forth movement as $F = \bar{F} = F_0$ as illustrated in Fig.~\ref{fig:example_fermat}(b). In this scenario, the approximate solution is $(252,117)$ as shown in Fig.~\ref{fig:example_fermat}(a). If this solution is applied in a real-world scenario with wind, the resulting transition time is 4039 seconds. When the wind effect is incorporated, the dynamic sets are defined as in (\ref{eq:example_fermat_back_and_forth_dynamic_sets}); see Fig.~\ref{fig:example_fermat}(d). The optimal solution then shifts to $(242,52)$, depicted in Fig.~\ref{fig:example_fermat}(c), reducing the objective value to 3960 seconds. Note that, since the sets $\Omega_i$ for $i = 1, 2, 3$ are not strictly convex, the uniqueness of solutions to this problem cannot be generally guaranteed.

Table~\ref{tab:Fermat-Toriceili-results} presents ten cases with varying settings and parameters to evaluate the efficiency of the generalized Fermat--Torricelli problem in accounting for wind effects. For simplicity, all the cases are conducted on a small scale, with the UAV's nominal speed set to $1 \text{(m/s)}$; these parameters can be appropriately scaled to reflect the real scenarios. The results indicate that, for single-UAV operations, the generalized Fermat--Torricelli problem reduces the transition time by approximately $3.3\%$.
\begin{figure}
\centering
\subfigure(a){\includegraphics[width=0.4\textwidth]
{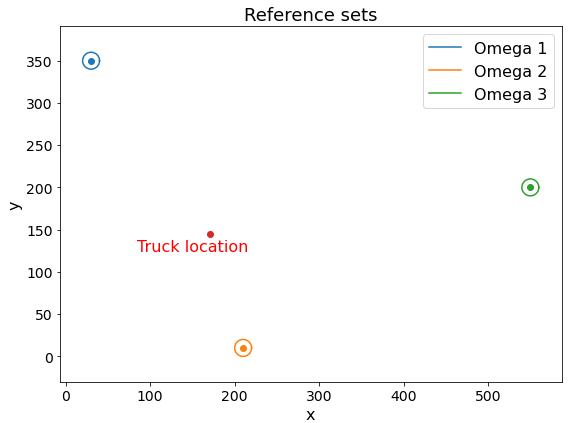}} 
\hspace{0.5 cm}
\subfigure(b){\includegraphics[width=0.4\textwidth]{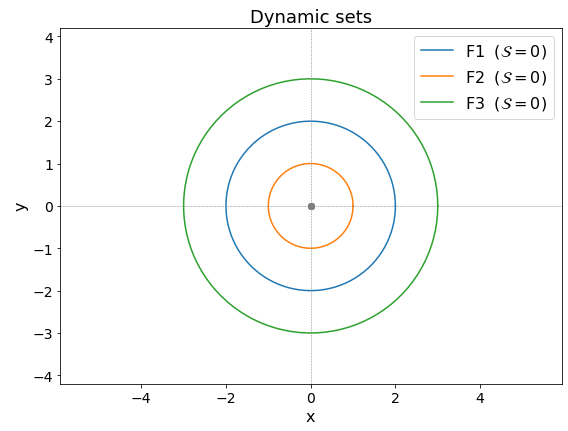}} 
\subfigure(c){\includegraphics[width=0.4\textwidth]
{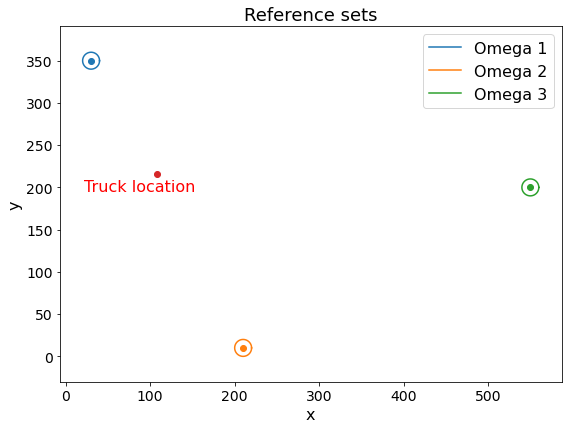}} 
\hspace{0.5 cm}
\subfigure(d){\includegraphics[width=0.4\textwidth]{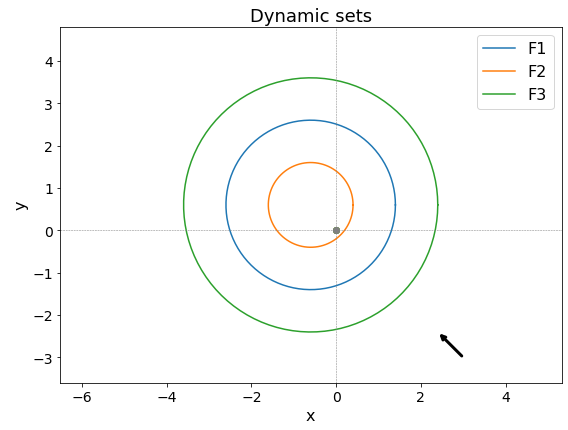}} 
\caption{Case 2 in Table~\ref{tab:Sylvester_results}:  
(a) Approximated optimal location of the truck when the wind effect is neglected.
(b) Dynamic sets with $\|\mathcal{S}\|=0$.  
(c) Approximated optimal location of the truck under windy conditions.  
(d) Dynamic sets of the UAVs under windy conditions with $\mathcal{S}=(-0.6,0.6)$.  
The black arrow indicates the wind direction}
\label{fig:example_Sylvester}
\end{figure}

\subsection{Applications of the Generalized Sylvester Problem}\label{sec:Multiple_UAV_example}
In the context of Section~\ref{sec:single_UAV_example}, we consider a scenario in which the rescue team deploys a heterogeneous fleet of three UAVs. These UAVs are launched at the earliest possible time after a disaster, and the goal is to determine an optimal location for a mobile station to collect them as quickly as possible. Timely retrieval of the UAVs enables rapid access to critical data, which is essential for effective response and decision-making in emergency situations. We consider this problem in the following setting corresponding to Case~2 from Table~\ref{tab:Sylvester_results}:
\begin{align*}
\Omega_0 &= \mathbb{R}^2, \\
\Omega_1 &= \Big\{(x,y) \in \mathbb{R}^2 \ \Big|\ \frac{(x-30)^2}{10^2}+\frac{(y-350)^2}{10^2}\le1\Big\}, \\
\Omega_2 &=\Big\{(x,y) \in \mathbb{R}^2 \ \Big|\ \frac{(x-210)^2}{10^2}+\frac{(y-10)^2}{10^2}\le1\Big\}, \\
\Omega_3 &=\Big\{(x,y)\in \mathbb{R}^2 \ \Big|\ \frac{(x-550)^2}{10^2}+\frac{(y-200)^2}{10^2}\le1\Bigg\}.
\end{align*}
For the dynamic sets, we have the following:
\begin{align*}
\ F_1 &=\big\{(f_1,f_2) \in \mathbb{R}^2 \ \big|\ (f_1+0.6)^2+(f_2-0.6)^2\le 2^2\big\}, \\
\ F_2 &=\big\{(f_1,f_2) \in \mathbb{R}^2 \ \big|\ (f_1+0.6)^2+(f_2-0.6)^2\le 1^2\big\}, \\
\ F_3 &=\big\{(f_1,f_2) \in \mathbb{R}^2 \ \big|\ (f_1+0.6)^2+(f_2-0.6)^2\le 3^2\big\}.
\end{align*}
Let us set 1(m/s) as the nominal speed of the slowest UAV, which is represented by $F_2$. Consequently, the rescue team is equipped with two additional UAVs that are twice and three times faster, respectively. In addition, consider the wind vector given by $\mathcal{S}=(-0.6,0.6)$.

To find an optimal location of the UAV station, we formulate this model as the following generalized Sylvester problem, which is a variant of the generalized SFT problem \eqref{eq:SFT}:
\begin{equation} \label{eq:example2_sylvester}
\min_{x\in \Omega_0} Y(x) \ \ \ \textnormal{with}\ \ \ Y(x):=\max\big\{\rho_{F_i}^{\Omega_i}(x) \ \big|\ \ i=1,2,3\big\}.
\end{equation}
Then Algorithm~\ref{alg:subgradient} allows us to determine an approximate minimizer of  (\ref{eq:example2_sylvester}). To proceed, we need to compute a subgradient of $Y(x)$ at each iteration $k$. Following the procedure described in Section~\ref{sec:single_UAV_example} gives us the calculation of $v_i^k \in \partial \rho_{F_i}^{\Omega_i}(x^k)$. Subsequently, $v^k \in \partial Y(x^k)$ is obtained  by using the formula derived from Theorem~\ref{theorem:max_rule}:
\begin{equation*}\label{eq:example_sylvester_subdifferential_formula}
\partial Y(x^k) = \operatorname{co}\left\{\bigcup_{i \in I(x^k)} \partial \rho_{F_i}^{\Omega_i}(x^k)\right\}.
\end{equation*}

\begin{table}[htbp]
\caption{Experimental Results for Generalized Sylvester Cases}
\vspace{0.25cm}
\begin{tabular}{ccc|cc|cc|cc}
\hline
\multicolumn{1}{c}{Case} &
\multicolumn{1}{c}{Parameters} &
\multicolumn{1}{c|}{Wind Vector}&
\multicolumn{2}{c|}{\cellcolor{gray!20} Wind Neglected} &
\multicolumn{2}{c|}{ \cellcolor{gray!20} Wind Included} &
\multicolumn{1}{c}{Absolute  }&
\multicolumn{1}{c}{Relative }\\ 
\multicolumn{1}{c}{No.} &
\multicolumn{1}{c}{(Table \ref{tab:experiments-settings})} &
\multicolumn{1}{c|}{$(\mathcal{S}=(s_1,s_2))$}&
\multicolumn{1}{c}{$x^\ast$} &
\multicolumn{1}{c|}{\(Z_N(s) \)} &
\multicolumn{1}{c}{$x^\ast$} &
\multicolumn{1}{c|}{\(Z_I(s) \)} &
\multicolumn{1}{c}{Imp(s)} &
\multicolumn{1}{c}{Imp($\%$)}
\\ \hline
1&Info 4&$(-0.7,0.7)$&(175,138)&242 &(95,228)&124&100&41.3\\
2&Info 4&$(-0.6,0.6)$&(175,138)&212 &(108,216)&124&88&41.5\\
3&Info 4&$(-0.5,0.5)$&(175,138)&190 &(116,202)&124&66&34.7\\
4&Info 4&$(-0.4,0.4)$&(175,138)&171 &(122,188)&124&47&27.4\\
5&Info 4&$(-0.3,0.3)$&(175,138)&156 &(146,179)&124&32&20.5\\
6&Info 4&$(-0.2,0.2)$&(175,138)&143 &(151,163)&123&20&13.9\\
7&Info 4&$(-0.1,0.1)$&(175,138)&133 &(175,139)&123&10&7.51\\
8&Info 4&(0.6,0.6)&(175,138)&163&(261,217)&125&38&23.3\\
9&Info 4&$(0.6,-0.6)$&(175,138)&721 &(265,64)&126&595&82.5\\
10&Info 4&$(-0.6,-0.6)$&(175,138)&481 &(118,67)&128&353&73.3\\
11&Info 5&$(0.6,-0.6)$&(239,163)&1401 &(371,40)&212&1189&84.8\\
12&Info 6&$(0.6,-0.6)$&(331,227)&525 &(386,140)&315&210&40.0\\
13&Info 6&$(-0.8,0)$&(331,227)&1156 &(161,124)&309&847&73.2\\
14&Info 6&$(0.8,0)$&(331,227)&432 &(474,311)&306&126&29.1\\
\hline
\end{tabular}
\vspace{0.1cm}
\label{tab:Sylvester_results}
\vspace{0.3cm}
{\scriptsize $Z_N$ and $Z_I$ denote the objective function values (transition time in seconds) for the cases ignoring and considering wind, respectively. Absolute Imp ($Z_N-Z_I$) and Relative Imp ($(Z_N-Z_I)/Z_N$) denote the absolute and relative improvements.}
\end{table}

Initially, we consider the case where the wind effect is ignored ($\mathcal{S}=0$) and the dynamic sets are illustrated in Fig.~\ref{fig:example_Sylvester}(b). The approximated solution in this case is $(175,138)$, as shown in Fig.~\ref{fig:example_Sylvester}(a), which corresponds to the transition time of $212$ seconds in the real scenario under windy conditions. When the wind effect is taken into account with $\mathcal{S}=(-0.6,0.6)$, the dynamic sets are modified as depicted in Fig.~\ref{fig:example_Sylvester}(d), which gives us the approximated optimal location $(108,216)$ as illustrated in Fig.~\ref{fig:example_Sylvester}(c). The transition time for $(108,216)$ is $124$ seconds, representing a $41.5\%$ reduction in transition time. This clearly demonstrates the effectiveness of the generalized Sylvester problem in addressing this challenge.

To evaluate the performance of the generalized Sylvester problem in mitigating wind effects, 14 cases with varying settings and parameters are considered as summarized in Table~\ref{tab:Sylvester_results}. For simplicity, all cases are performed on a small scale, assigning $1$(m/s) to the slowest UAV nominal speed; these values can be scaled to represent real-world scenarios. The results demonstrate that the generalized Sylvester problem significantly reduces transition time, achieving an approximate $84\%$ improvement. In Cases~1--7 of Table~\ref{tab:Sylvester_results}, all parameters are identical except for the length of the wind vector. The results of these cases are illustrated in Fig. \ref{fig:Wind_effect}(b), which highlights the impact of wind speed. As expected, applying the generalized Sylvester model at higher wind speeds resulted in greater efficiency and a larger reduction in transition time.

\subsection{Applications of the Generalized SFT Problem}\label{sec:Multiple_Example_MSMG}

In the setting described in Section~\ref{sec:Multiple_UAV_example}, the rescue team first determines the optimal location for the mobile station and then deploys a fleet of three UAVs. After completing their missions, the UAVs return to the mobile station for data transfer and recharging. To address this challenge while accounting for wind conditions and fleet heterogeneity (corresponding to Case~1 in Table~\ref{tab:SFT_results}), we employ the following generalized SFT model:
\begin{equation}\label{eq:Example3_SFT}
\min_{x\in\Omega_0} S(x),\;\mbox{ where }\;S(x):=\max\big\{\rho_{F_i}^{\Omega_i}(x)+\rho_{\Bar{F}_i}^{\Omega_i}(x)\ \big|\ i=1,2,3\big\}.
\end{equation}
Assuming that $\Bar{F}_i=-F_i$, compute \(\partial S(x)\) by applying Theorems~\ref{theorem:max_rule} and~\ref{theorem:sum_rule_subdifferential}. This gives us
\begin{equation}
\partial S(x) = \operatorname{co} \left\{ \bigcup_{i \in I(x)} \left( \partial \rho_{F_i}^{\Omega_i}(x) + \partial \rho_{\bar{F}_i}^{\Omega_i}(x) \right) \,\middle|\, i = 1, 2, 3 \right\},
\end{equation}
\begin{figure}
\centering
\subfigure(a){\includegraphics[width=0.4\textwidth]
{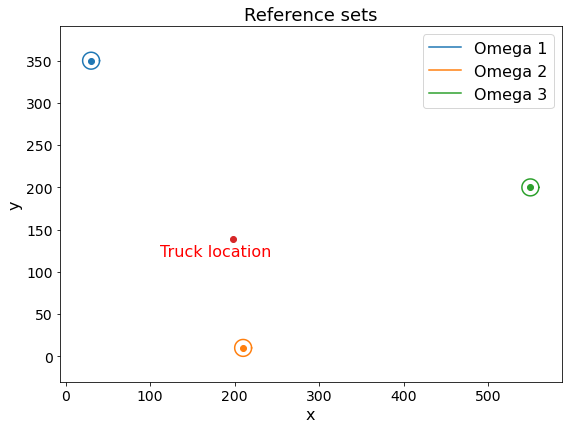}} 
\hspace{0.5 cm}
\subfigure(b){\includegraphics[width=0.4\textwidth]{No_Wind_Dynamic.png}} 
\subfigure(c){\includegraphics[width=0.4\textwidth]
{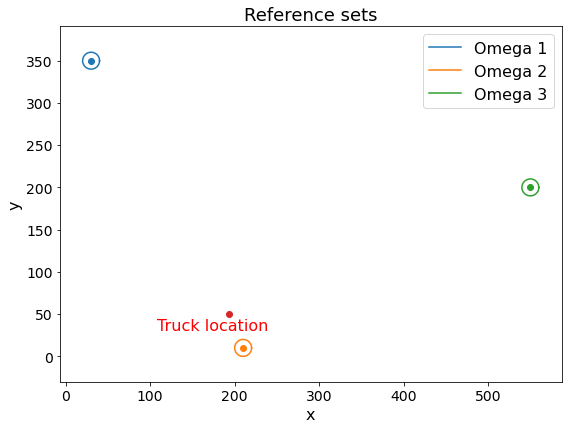}} 
\hspace{0.5 cm}
\subfigure(d){\includegraphics[width=0.4\textwidth]{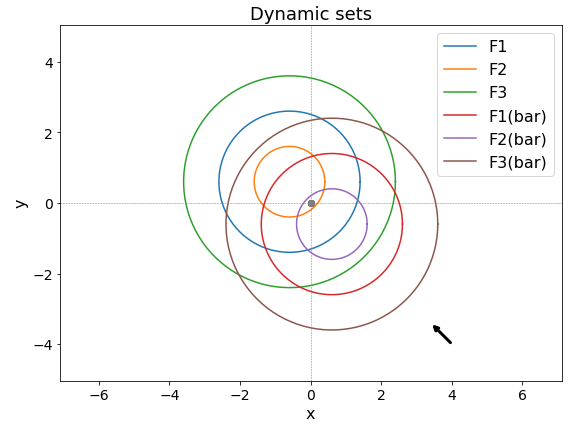}} 
\caption{Case 1 in Table~\ref{tab:SFT_results}:  
(a) Approximated optimal location of the truck when the wind effect is neglected.
(b) Dynamic sets with $\|\mathcal{S}\|=0$.  
(c) Approximated optimal location of the truck under windy conditions.  
(d) Dynamic sets of the UAVs under windy conditions with $\mathcal{S}=(-0.6,0.6)$.  
The black arrow indicates the wind direction.}
\label{fig:example_SFT}
\end{figure}
where $\partial \rho_{F_i}^{\Omega_i}(x)$ and $\partial \rho_{\Bar{F}_i}^{\Omega_i}(x)$ are calculated similarly to 
Section~\ref{sec:single_UAV_example}. The subgradient algorithm (Algorithm~\ref{alg:subgradient}) can be utilized to compute an optimal solution to problem~\eqref{eq:Example3_SFT}. In continuation, we consider a more challenging case of this scenario, which has a real-world application. In this case, UAVs are deployed from a point \(x\) (the location of the mobile station), first travel to the nearest point in \(\Omega_i\), and then return from the farthest point in \(\Omega_i\) back to \(x\). Cases of this type commonly arise when UAVs are used to cover regions \( \Omega_i \) for tasks 
such as data collection or aerial photography of affected areas, employing coverage path planning algorithms (see, for example, \cite{kazemdehbashi2025algorithm, Coombes2018, Choset2000, Latombe1991}). Then, in the worst-case situation, UAVs complete their coverage at the farthest point and then must return to the mobile station for recharging and data transfer. This can be formulated via the extended version of the generalized SFT problem (\ref{eq:SFT-extension}) as follows:
\begin{equation}\label{eq:extended_SFT_Multiple_UAVs}
\min_{x\in \Omega_0} \Hat{S}(x),\;\mbox{ where }\;
\Hat{S}(x):=\max\big\{\rho_{\Bar{F}_i}^{\Omega_i}(x)+\Bar{\rho}_{F_i}^{\Omega_i}(x) \ \big|\ i=1,2,3\big\}.
\end{equation}
Assuming again that $\Bar{F}_i=-F_i$ gives us by Theorems~\ref{theorem:max_rule} and~\ref{theorem:sum_rule_subdifferential} that
\begin{equation}
\label{eq:S_hat_subdifferential}
\partial \Hat{S}(x) = \operatorname{co} \left\{ \bigcup_{i \in I(x)} \left( \partial \rho_{\Bar{F}_i}^{\Omega_i}(x) + \partial \Bar{\rho}_{F_i}^{\Omega_i}(x) \right) \,\middle|\, i = 1, 2, 3\right\}.
\end{equation}
Now we use Algorithm~\ref{alg:subgradient} to find an approximate minimizer of this problem. At each step of the algorithm, a subgradient from $\partial \Hat{S}(x)$ is required. To this end, we compute \( \partial \rho_{\Bar{F}_i}^{\Omega_i}(x) \) similarly to Section~\ref{sec:single_UAV_example}. To simplify the procedure, the lower estimate of $\partial\Bar{\rho}_{F_i}^{\Omega_i}(x)$ from Theorem~\ref{theorem:MSMG_properties}  can be used instead of the full calculation of the subdifferential. In this way, we get
\[
\nabla \rho_{F_i}(x - \omega^\prime_i) \subset \partial \Bar{\rho}_{F_i}^{\Omega_i}(x)\;\mbox{ with }\;\omega_i^\prime \in \Bar{\Pi}_{F_i}(x;\Omega_i)
\]
Then  determining a subgradient of $\partial \Hat{S}(x)$ becomes straightforward from (\ref{eq:S_hat_subdifferential}).

\begin{figure}
\centering
\subfigure(a)
{\includegraphics[width=0.37\textwidth]{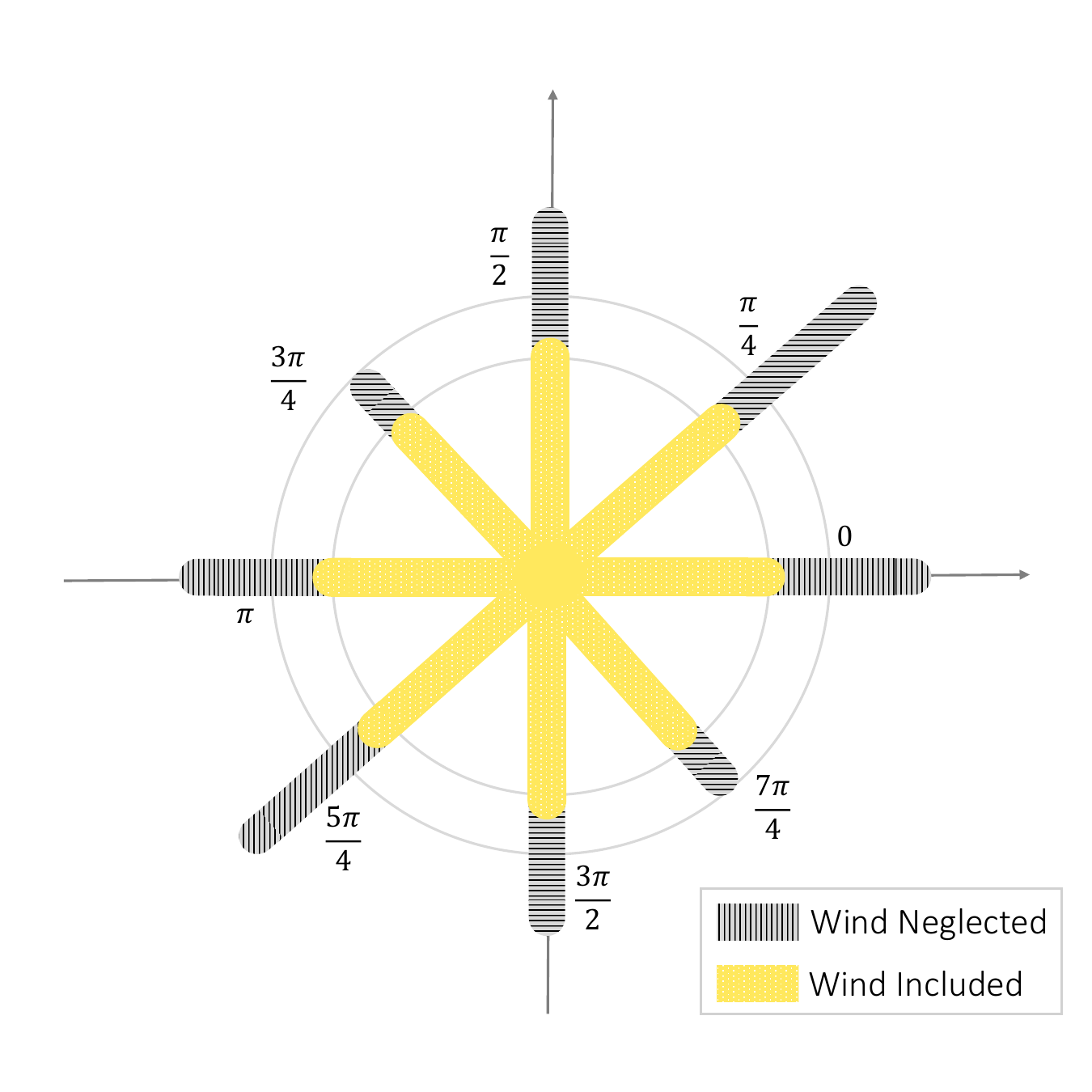}} 
\hspace{0.05 cm}
\subfigure(b)
{\includegraphics[width=0.53\textwidth]{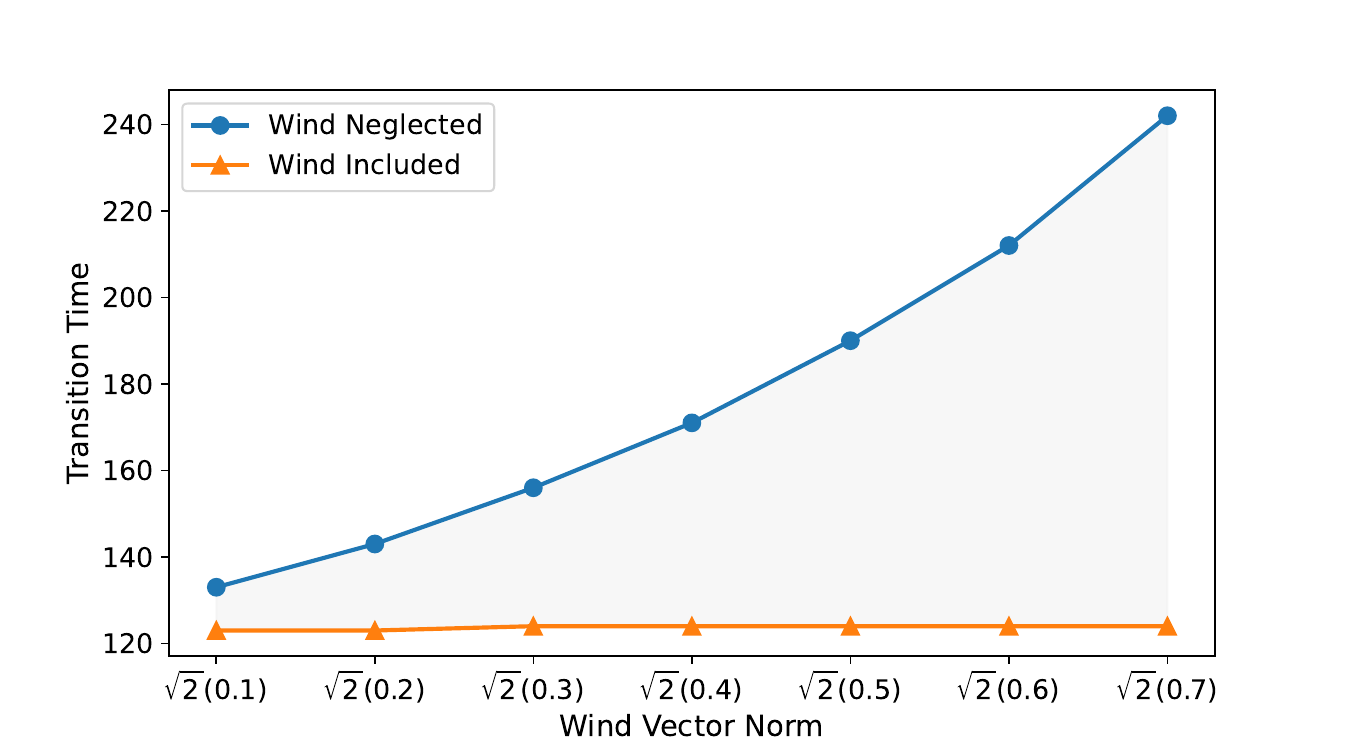}} 
\caption{(a) Effect of wind direction on transition time at approximately similar wind speeds. The gray line shows transition time without considering wind, while the yellow line reflects results from the generalized SFT model accounting for wind. Results correspond to Cases 14–21 in Table \ref{tab:SFT_results}. (b) Impact of wind vector norm (speed) on transition time. Results correspond to Cases 1–7 in Table \ref{tab:Sylvester_results}.}
\label{fig:Wind_effect}
\end{figure}\vspace*{0.03in}

If in the case modeled by problem  (\ref{eq:extended_SFT_Multiple_UAVs}), UAVs are deployed at the earliest time and the objective is to determine an optimal location of the mobile station for collecting the data as described in 
Section~\ref{sec:Multiple_UAV_example}, the extended generalized Sylvester problem (\ref{eq:extention-sylvester}) can be used to optimize the location of the mobile station. 
Furthermore, if the rescue team employs a single UAV instead of three (as in the scenario presented in Section~\ref{sec:single_UAV_example}), the UAV travels from the mobile station to the closest point of $\Omega_i$ and returns from the farthest point while repeating this pattern for the subsequent reference sets. The extended generalized Fermat-Torricelli problem (\ref{eq:eq:Extended-Generalized-Fermat-Model}) can then be used to determine the optimal location of the mobile station. 

Table~\ref{tab:SFT_results} summarizes 24 cases with different parameter settings to evaluate the efficiency of the generalized SFT problem (extended variant) in accounting for wind effects. For simplicity, all cases are conducted on a small scale, with the slowest UAV assigned a nominal speed of $1$(m/s); these values can be appropriately scaled to represent real-world scenarios. The results show that the generalized SFT problem (extended variant) reduces the transition time by approximately $66.8\%$. Furthermore, Cases~14--21 share identical settings except for the direction of the wind vector. Fig.~\ref{fig:Wind_effect}(a) illustrates these results that confirm the ability of the generalized SFT formulation to address wind effects across varying directions. 

\begin{table}
\caption{Experimental Results for Generalized SFT Cases (Extended Version)}
\vspace{0.25cm}
\begin{tabular}{ccc|cc|cc|cc}
\hline
\multicolumn{1}{c}{Case} &
\multicolumn{1}{c}{Parameters} &
\multicolumn{1}{c|}{Wind Vector}&
\multicolumn{2}{c|}{\cellcolor{gray!20} Wind Neglected} &
\multicolumn{2}{c|}{ \cellcolor{gray!20} Wind Included} &
\multicolumn{1}{c}{Absolute  }&
\multicolumn{1}{c}{Relative }\\ 
\multicolumn{1}{c}{No.} &
\multicolumn{1}{c}{(Table \ref{tab:experiments-settings})} &
\multicolumn{1}{c|}{$(\mathcal{S}=(s_1,s_2))$}&
\multicolumn{1}{c}{$x^\ast$} &
\multicolumn{1}{c|}{\(Z_N(s) \)} &
\multicolumn{1}{c}{$x^\ast$} &
\multicolumn{1}{c|}{\(Z_I(s) \)} &
\multicolumn{1}{c}{Imp(s)} &
\multicolumn{1}{c}{Imp($\%$)}
\\ \hline
1&Info 4&$(-0.6,0.6)$&(171,134)&905 &(189,56)&401&504&55.6\\
2&Info 4&$(0.6,0.6)$&(171,134)&637 &(175,78)&342&295&46.3\\
3&Info 4&$(-0.6,-0.6)$&(171,134)&581 &$(181,85)$&340&241&41.4\\
4&Info 4&$(0.6,-0.6)$&(171,134)&798 &(214,88)&394&404&50.6\\
5&Info 4&$(0.85,0.0)$&(171,134)&528 &(183,91)&344&184&34.8\\
6&Info 4&$(0.0,0.85)$&(171,134)&961 &(160,50)&389&572&59.5\\
7&Info 4&$(-0.85,0.0)$&(171,134)&565 &(188,88)&348&217&38.4\\
8&Info 4&$(0.0,-0.85)$&(171,134)&844 &(156,62)&384&460&54.5\\
9&Info 5&(0.4,0.8)&(351,234)&1842 &(300,90)&610&1232&66.8\\
10&Info 5&$(0.1,0.1)$&(351,234)&450 &$(333,223)$&441&9&2.0\\
11&Info 5&$(-0.1,-0.1)$&(351,234)&447 &(347,231)&444&3&0.0\\
12&Info 5&$(0.1,-0.1)$&(351,234)&448 &$(336,225)$&443&5&1.11\\
13&Info 5&$(-0.1,0.1)$&(351,234)&453 &(340,225)&444&9&1.98\\
14&Info 6&$(0.8,0)$&(333,220)&1442 &(234,138)&829&613&42.5\\
15&Info 6&$(0.56,0.56)$&(333,220)&1648 &(225,124)&877&771&46.7\\
16&Info 6&$(0,0.8)$&(333,220)&1364 &(246,136)&819&545&39.9\\
17&Info 6&$(-0.56,0.56)$&(333,220)&981 &(269,178)&734&247&25.1\\
18&Info 6&$(-0.8,0)$&(333,220)&1375 &(245,141)&822&553&40.2\\
19&Info 6&$(-0.56,-0.56)$&(333,220)&1564 &(235,135)&869&695&44.4\\
20&Info 6&$(0,-0.8)$&(333,220)&1307&(235,153)&819&488&37.3\\
21&Info 6&$(0.56,-0.56)$&(333,220)&987 &(267,178)&734&253&25.6\\
22&Info 7&$(0.69,0.4)$&(2265,1459)&4148&(2805,1335)&2028&2120&51.1\\
23&Info 7&$(0.56,0.56)$&(2265,1459)&3840&(2878,1295)&2064&1776&46.2\\
24&Info 7&$(0.8,0.0)$&(2265,1459)&4117&(2673,1456)&1955&2162&52.5\\
\hline
\end{tabular}
\vspace{0.1cm}
\label{tab:SFT_results}
\vspace{0.3cm}
{\scriptsize $Z_N$ and $Z_I$ denote the objective function values (transition time in seconds) for the cases ignoring and considering wind, respectively. Absolute Imp ($Z_N-Z_I$) and Relative Imp ($(Z_N-Z_I)/Z_N$) denote the absolute and relative improvements.}
\end{table}

For problem  (\ref{eq:extended_SFT_Multiple_UAVs}) associated with Case~1 in 
Table~\ref{tab:SFT_results}, the first step is to identify an approximate optimal location of the UAV station under the condition $\|\mathcal{S}\|=0$ with the corresponding dynamic sets shown in Fig.~\ref{fig:example_SFT}(b). The obtained solution is $(171,134)$ as shown in Fig.~\ref{fig:example_SFT}(a) with the corresponding transition time of $905\,\text{s}$ under real conditions but without taking wind into account. When incorporating the wind vector $\mathcal{S}=(-0.6,0.6)$, the dynamic sets change accordingly as illustrated in Fig.~\ref{fig:example_SFT}(d). In this case, the approximate optimal location of the UAV station shifts to $(189,56)$ as shown in Fig.~\ref{fig:example_SFT}(c), with a reduced transition time of $401\,\text{s}$. This represents the $55.6\%$ improvement and thus demonstrates the effectiveness of the generalized SFT problem (extended variant) in addressing weather uncertainty. The obtained results highlight that neglecting wind effects can result in a significant time loss of nearly $500\,\text{s}$, which is an inefficiency that may be critical in time-sensitive rescue operations.  \vspace*{0.05in}

Finally in this section, we provide information about the reference sets and dynamic sets used in the experiments as summarized in Table \ref{tab:experiments-settings}.

\begin{table}[h]
\caption{Experimental Information and Parameters}
	\vspace{0.25cm}
	\begin{tabular}{c|c|c|}
		\hline
		\multicolumn{1}{c|}{No.} &
		\multicolumn{1}{c|}{Refrence sets} &
		\multicolumn{1}{c|}{Dynamic sets} \\ 
		\multicolumn{1}{c|}{} &
		\multicolumn{1}{c|}{} &
		\multicolumn{1}{c|}{}\\ \hline
		
		Info 1& $\Omega_0 = \mathbb{R}^2$,
		
		& 	\\	
		& $\Omega_1 = \left\{(x, y) \in \mathbb{R}^2 \mid \max\{|x - 30|, |y - 350|\} \le 15 \right\}$,
		
		& $F_1 = \left\{(f_1, f_2) \in \mathbb{R}^2 \mid (f_1 - s_1)^2 + (f_2 - s_2)^2 \le 1 \right\}$,\\
		
		& $\Omega_2 = \left\{(x, y) \in \mathbb{R}^2 \mid \max\{|x - 210|, |y - 10|\} \le 15 \right\}$,
		
		& $F_2 = \left\{(f_1, f_2) \in \mathbb{R}^2 \mid (f_1 - s_1)^2 + (f_2 - s_2)^2 \le 1 \right\}$,\\
		
		& $\Omega_3 = \left\{(x, y) \in \mathbb{R}^2 \mid \max\{|x - 550|, |y - 200|\} \le 15 \right\}$
		
		& $F_3 = \left\{(f_1, f_2) \in \mathbb{R}^2 \mid (f_1 - s_1)^2 + (f_2 - s_2)^2 \le 1 \right\}$\\
		
		\hline

		Info 2& $\Omega_0 = \mathbb{R}^2$,
		
		& \\
		
		& $\Omega_1 = \left\{(x, y) \in \mathbb{R}^2 \mid \max\{|x - 30|, |y - 750|\} \le 20 \right\}$,
		
		& $F_1 = \left\{(f_1, f_2) \in \mathbb{R}^2 \mid (f_1 - s_1)^2 + (f_2 - s_2)^2 \le 1 \right\}$,\\
		
		& $\Omega_2 = \left\{(x, y) \in \mathbb{R}^2 \mid \max\{|x - 550|, |y - 10|\} \le 10 \right\}$,
		
		& $F_2 = \left\{(f_1, f_2) \in \mathbb{R}^2 \mid (f_1 - s_1)^2 + (f_2 - s_2)^2 \le 1 \right\}$,\\
		
		& $\Omega_3 = \left\{(x, y) \in \mathbb{R}^2 \mid \max\{|x - 950|, |y - 400|\} \le 15 \right\}$
		
		& $F_3 = \left\{(f_1, f_2) \in \mathbb{R}^2 \mid (f_1 - s_1)^2 + (f_2 - s_2)^2 \le 1 \right\}$\\
		
		\hline
		
		Info 3& $\Omega_0 = \mathbb{R}^2$,
		
		& \\
		
		& $\Omega_1 = \left\{(x, y) \in \mathbb{R}^2 \mid \max\{|x - 30|, |y - 200|\} \le 5 \right\}$,
		
		& $F_1 = \left\{(f_1, f_2) \in \mathbb{R}^2 \mid (f_1 - s_1)^2 + (f_2 - s_2)^2 \le 1 \right\}$,\\
		
		& $\Omega_2 = \left\{(x, y) \in \mathbb{R}^2 \mid \max\{|x - 150|, |y - 10|\} \le 5 \right\}$,
		
		& $F_2 = \left\{(f_1, f_2) \in \mathbb{R}^2 \mid (f_1 - s_1)^2 + (f_2 - s_2)^2 \le 1 \right\}$,\\
		
		& $\Omega_3 = \left\{(x, y) \in \mathbb{R}^2 \mid \max\{|x - 350|, |y - 90|\} \le 5 \right\}$
		
		& $F_3 = \left\{(f_1, f_2) \in \mathbb{R}^2 \mid (f_1 - s_1)^2 + (f_2 - s_2)^2 \le 1 \right\}$\\
		\hline
		
		Info 4& $\Omega_0 = \mathbb{R}^2$,
		
		& \\
		
		& $\Omega_1 = \{(x,y) \in \mathbb{R}^2 \ |\ \frac{(x-30)^2}{10^2}+\frac{(y-350)^2}{10^2}\le1\},$
		
		&  $F_1 = \{(f_1,f_2) \in \mathbb{R}^2 \ |\ (f_1-s_1)^2+(f_2-s_2)^2\le 2^2\},$\\
		
		& $\Omega_2 = \{(x,y) \in \mathbb{R}^2 \ |\ \frac{(x-210)^2}{10^2}+\frac{(y-10)^2}{10^2}\le1\},$
		
		&  $F_2 = \{(f_1,f_2) \in \mathbb{R}^2 \ |\ (f_1-s_1)^2+(f_2-s_2)^2\le 1^2\},$\\
		
		& $\Omega_3 = \{(x,y) \in \mathbb{R}^2 \ |\ \frac{(x-550)^2}{10^2}+\frac{(y-200)^2}{10^2}\le1\}$
		
		&  $F_3 = \{(f_1,f_2) \in \mathbb{R}^2 \ |\ (f_1-s_1)^2+(f_2-s_2)^2\le 3^2\}$\\
		
		\hline
		
		Info 5& $\Omega_0 = \mathbb{R}^2$,
		
		& \\
		
		& $\Omega_1 = \{(x,y) \in \mathbb{R}^2 \ |\ \frac{(x-40)^2}{10^2}+\frac{(y-550)^2}{10^2}\le1\},$
		
		&  $F_1 = \{(f_1,f_2) \in \mathbb{R}^2 \ |\ (f_1-s_1)^2+(f_2-s_2)^2\le 2^2\},$\\
		
		& $\Omega_2 = \{(x,y) \in \mathbb{R}^2 \ |\ \frac{(x-410)^2}{10^2}+\frac{(y-20)^2}{10^2}\le1\},$
		
		&  $F_2 = \{(f_1,f_2) \in \mathbb{R}^2 \ |\ (f_1-s_1)^2+(f_2-s_2)^2\le 1^2\},$\\
		
		& $\Omega_3 = \{(x,y) \in \mathbb{R}^2 \ |\ \frac{(x-750)^2}{10^2}+\frac{(y-350)^2}{10^2}\le1\}$
		
		&  $F_3 = \{(f_1,f_2) \in \mathbb{R}^2 \ |\ (f_1-s_1)^2+(f_2-s_2)^2\le 3^2\}$\\
		
		\hline
		
		Info 6& $\Omega_0 = \mathbb{R}^2$,
		
		& \\
		
		& $\Omega_1 = \{(x,y) \in \mathbb{R}^2 \ |\ \frac{(x-40)^2}{10^2}+\frac{(y-550)^2}{10^2}\le1\},$
		
		&  $F_1 = \{(f_1,f_2) \in \mathbb{R}^2 \ |\ (f_1-s_1)^2+(f_2-s_2)^2\le 2^2\},$\\
		
		& $\Omega_2 = \{(x,y) \in \mathbb{R}^2 \ |\ \frac{(x-110)^2}{10^2}+\frac{(y-20)^2}{10^2}\le1\},$
		
		&  $F_2 = \{(f_1,f_2) \in \mathbb{R}^2 \ |\ (f_1-s_1)^2+(f_2-s_2)^2\le 1^2\},$\\
		
		& $\Omega_3 = \{(x,y) \in \mathbb{R}^2 \ |\ \frac{(x-650)^2}{10^2}+\frac{(y-150)^2}{10^2}\le1\},$
		
		&  $F_3 = \{(f_1,f_2) \in \mathbb{R}^2 \ |\ (f_1-s_1)^2+(f_2-s_2)^2\le 3^2\},$\\
		
		& $\Omega_4 = \{(x,y) \in \mathbb{R}^2 \ |\ \frac{(x-750)^2}{10^2}+\frac{(y-650)^2}{10^2}\le1\}$
		
		&  $F_4 = \{(f_1,f_2) \in \mathbb{R}^2 \ |\ (f_1-s_1)^2+(f_2-s_2)^2\le 2^2\}$\\
		
		\hline
		
		Info 7& $\Omega_0 = \mathbb{R}^2$,
		
		& \\
		
		& $\Omega_1 = \{(x,y) \in \mathbb{R}^2 \ |\ \frac{(x-10)^2}{10^2}+\frac{(y-2000)^2}{10^2}\le1\},$
		
		&  $F_1 = \{(f_1,f_2) \in \mathbb{R}^2 \ |\ (f_1-s_1)^2+(f_2-s_2)^2\le 3^2\},$\\
		
		& $\Omega_2 = \{(x,y) \in \mathbb{R}^2 \ |\ \frac{(x-3000)^2}{10^2}+\frac{(y-1700)^2}{10^2}\le1\},$
		
		&  $F_2 = \{(f_1,f_2) \in \mathbb{R}^2 \ |\ (f_1-s_1)^2+(f_2-s_2)^2\le 1^2\},$\\
		
		& $\Omega_3 = \{(x,y) \in \mathbb{R}^2 \ |\ \frac{(x-1700)^2}{10^2}+\frac{(y-20)^2}{10^2}\le1\}$
		
		&  $F_3 = \{(f_1,f_2) \in \mathbb{R}^2 \ |\ (f_1-s_1)^2+(f_2-s_2)^2\le 2^2\}$\\
        \hline
	\end{tabular}
	\vspace{0.1cm}
	
	\label{tab:experiments-settings}
	\vspace{0.1cm}
	{\scriptsize Note that the dynamic sets in this table have two parameters, $s_1$ and $s_2$, which are obtained from the vector $\mathcal{S} = (s_1, s_2)$ under the Wind Vector column in Tables~\ref{tab:Fermat-Toriceili-results}, \ref{tab:Sylvester_results}, and \ref{tab:SFT_results}.
	}
\end{table}
\clearpage

\section{Conclusion and Future Research}\label{sec:conc}

After a disaster such as, e.g., an earthquake, communication networks play a crucial role in supporting rescue operations. However, these networks are often partially or completely destroyed. In such scenarios, UAVs can serve as valuable tools, which offer large-scale WiFi coverage or DTNs. To enable this, a fleet of UAVs can be deployed from a mobile station (e.g., a truck). Given the operational constraints of UAVs---such as battery limitation and the influence of external forces like wind---it becomes essential to determine the optimal location for the mobile station that minimizes both time and energy consumption. This challenge can be framed as a single facility location problem, which connects closely to two well-known models in mathematical literature: the generalized Sylvester and Fermat--Torricelli problems.

The classical Sylvester and Fermat--Torricelli problems have long served as fundamental models in geometry and optimization. In this paper, we extended these problems into a more general framework that can include two types of gauge functions with several dynamic sets in \( \mathbb{R}^q \), capturing the complexities of real-world facility location challenges. We introduced the generalized \textit{Sylvester--Fermat--Torricelli (SFT)} problem, a novel framework for modeling single-facility location tasks in the presence of heterogeneous vehicle speeds, external forces like wind, and multiple distance norms (e.g., \( \ell_1 \) and \( \ell_2 \)).

Several cases were studied involving multiple UAVs, and we demonstrated how the generalized SFT model effectively addresses fleet heterogeneity and environmental conditions. Our results showed that this approach can reduce wasted operational time by up to 84\%, highlighting its practical efficiency. Overall, our findings demonstrate the versatility of the generalized SFT problem in modeling practical optimization scenarios. Moreover, our proposed model can serve as a foundational framework for modern facility location problems—particularly in UAV applications—offering a basis for further development in future studies. 

{\em Future research} will explore further generalizations, including {\em stochastic} elements and {\em multiobjective} formulations, to broaden the scope and applicability of this foundational model. Furthermore, other scenarios—such as {\em multifacility location problems} and cases involving the $\ell_1$ norm—can be considered in future studies. In addition, integrating the generalized SFT problem with {\em path planning algorithms} can be investigated for last-mile delivery applications to minimize time, energy consumption, and operational costs.

\section*{Acknowledgments}
Research of the first and second authors was partly supported by the National Science Foundation via CMMI Grant 1944068. Research of the third author was partly supported by the US National Science Foundation under grant DMS-2204519 and by the Australian Research Council under Discovery Project DP250101112.

\section{Appendix: Auxiliary Statements and Proofs}\label{appendix}

In the appendix, we present some additional material of its own interest, which is 
used to verify the main results above. We begin with providing some properties of  the generalized projection defined in (\ref{eq:generalized_projection}).\vspace*{0.05in}

\begin{proposition}\label{prop:generalized_projection_nonempty}
Let $F\subset \mathbb{R}^q$ be compact and convex set with $0\in \textnormal{int}\,F$, and let $\Omega \subset \mathbb{R}^q$ be nonempty, closed, and convex. Then we have $\Pi_F(x;\Omega)\neq\emptyset$ for all $x\in\mathbb{R}^q$. 
\end{proposition}
\begin{proof}
Consider the sublevel set of the function $\rho_F^\O$ given by
\[
\mathcal{L}_\lambda(\rho_F^\O) := \big\{\omega \in \Omega \;\big|\;\rho_F^\O(x - \omega) \le \lambda\big\}, \quad \lambda \in \mathbb{R}.
\]  
By Proposition~\ref{prop:proof_of_Lipshitz}, the function \( \rho_F^\O\) is continuous, and hence \( \mathcal{L}_\lambda(\rho_F^\O) \) is a closed set. Since \( \rho_F^\O(x - \omega) \to \infty \) as \( \omega \to \infty \), we conclude that \( \mathcal{L}_\lambda(\rho_F^\O) \) is also bounded, and hence $\mathcal{L}_\lambda(\rho_F^\O)$ is a compact subset of \( \mathbb{R}^q \). Therefore, the nonemptiness of 
generalized projection $\Pi_F(x;\Omega)$ follows from the classical Weierstrass existence theorem.
\end{proof}

The next proposition shows that a point belonging to the generalized projection of any $\ox\notin\O$ lies on the boundary of the reference set.\vspace*{0.05in}

\begin{proposition}\label{prop:projection&border} In addition to the assumptions of 
Proposition~{\rm\ref{prop:generalized_projection_nonempty}}, suppose that $\bar{x} \notin\Omega$. Then for any point $\bar{\omega} \in \Pi_F(\bar{x};\Omega)$, we get that $\bar{\omega} \in \textnormal{bd}\,\Omega$.
\end{proposition}
\begin{proof}
Arguing by contradiction, suppose that $\bar{\omega} \in \textnormal{int}\,\Omega$. Then there exists $\epsilon>0$ such that $\mathbb{B}(\bar{\omega};\epsilon)\subset \Omega$, where $\mathbb{B}(\bar{\omega};\epsilon)$ stands for the ball centered at $\bar\omega$ with radius $\ve>0$. Denoting
$y:=\bar{\omega}+\epsilon\frac{\bar{x}-\bar{\omega}}{\|\bar{x}-\bar{\omega}\|}$, we have the equalities
\[
\rho_F\Big(\bar{x}-y)=\rho_F(\bar{x}-\bar{\omega}-\epsilon \frac{\bar{x}-\bar{\omega}}{\|\bar{x}-\bar{\omega}\|}\Big)=\rho_F\Big((\bar{x}-\bar{\omega}\Big)(1-\frac{\epsilon}{\|\bar{x}-\bar{\omega}\|}\Big)\Big)=(1-\frac{\epsilon}{\|\bar{x}-\bar{\omega}\|})\rho_F(\bar{x}-\bar{\omega}).
\]
Since $\big(1-\frac{\epsilon}{\|\bar{x}-\bar{\omega}\|}\big)\rho_F(\bar{x}-\bar{\omega})<\rho_F(\bar{x}-\bar{\omega})$, this contradicts the definition of the generalized projection in \eqref{eq:generalized_projection}, and therefore $\bar{\omega}$ should be a boundary point of $\O$.
\end{proof}

Now we present sufficient conditions under which the inequality in the subadditivity property of the classical Minkowski function $\rho_F$  becomes an equality. \vspace*{0.05in}

\begin{proposition}\label{prop:equality_minkowski_strictly_convex}
Let $F$ be a compact and strictly convex subset of $\mathbb{R}^q$ with $0 \in \textnormal{int}\,F$. Then
\begin{equation}
\rho_F(x+y)=\rho_F(x)+\rho_F(y)\;\mbox{ whenever }\; x,y \ne0
\end{equation}
if and only if there exists $\lm>0$ such that $x=\lambda y$.
\end{proposition}
\begin{proof} It can be found in \parencite[Proposition~8.13]{mordukhovich2023easy}.
\end{proof}

The following lemma gives us sufficient conditions ensuring the equality in definition \eqref{eq:r_enlarge_set_based_minkowski} of the reference set enlargement.\vspace*{0.05in}

\begin{lemma}\label{lemma:r_enlargement_other_shape} Let F be a closed  and convex set with $0 \in \textnormal{int}\,F$, and let $\Omega \subset \mathbb{R}^q$ be an arbitrary nonempty set. Then we have the equality $\Omega+rF=\Omega_r$ for any $r > 0$.
\end{lemma}
\begin{proof} Whenever $x \in \Omega+rF$ and $r>0$, we have $\rho_F^\Omega(x) \le r$ and hence $x \in \Omega_r$. For the reverse inclusion, picking $x \in \Omega_r$ gives us $\rho_F^\Omega(x)\le r$. Taking $t \le r$ and $\rho_F^\Omega(x)=t$, we get $x \in \Omega+tF$. Since $F$ with is convex with  $0 \in \textnormal{int}\,F$, it follows that
\[
\frac{t}{r}F \subset F \Longrightarrow \Omega+tF \subset \Omega+rF,
\]
which tells us therefore that $x \in \Omega+rF$ and thus completes the proof.
\end{proof}

The next two lemmas establish the needed properties of the set-based Minkowski gauge $\rho_F^{\Omega}$.\vspace*{0.05in}

\begin{lemma}\label{lemma:equality_mikowski_to_omega_r} Let $F$ and $\O$ be a nonempty subsets of $\mathbb{R}^q$. Assume that $F$ is closed and convex and take $x \not \in \Omega_r$ with some $r>0$ and $\rho_F^\Omega(x)<\infty$. Then we have the representation
\begin{equation}\label{r-enlar}
\rho_F^\Omega(x)=\rho_F^{\Omega_r}(x)+r.
\end{equation}
\end{lemma}
\begin{proof}
Fix $\epsilon>0$ and take $t$ with \(\rho_F^{\Omega_r}(x) \leq t <\rho_F^{\Omega_r}(x) + \epsilon\). Then 
there exist \(u \in \Omega_r\) and \(f \in F\) with \(u + tf = x\). Since \(\rho_F^{\Omega}(u) \leq r\), we can find $\alpha>0$, \(\omega \in \Omega\), and \(f' \in F\) such that $\rho_F^{\Omega}(u)\le t^\prime<\rho_F^{\Omega}(u)+\alpha$ and \(\omega + t^\prime f' = u\). Substituting the latter into the expression for \(x\) gives us the equality
$
x = \omega + t^\prime f' + tf$. Employing further the convexity of \(F\), we have \(t^\prime f' + tf \in t^\prime F + tF = (t^\prime + t)F\), and hence $x \in \omega + (t^\prime + t)F$. This yields \(\rho_F^\Omega(x) \leq t^\prime + t < \rho_F^{\Omega}(u) + \rho_F^{\Omega_r}(x) +\alpha + \epsilon\). Passing to the limit as \(\epsilon\dn 0\) and $\alpha\dn 0$ brings us to
\[
\rho_F^\Omega(x) \leq \rho_F^{\Omega}(u)+\rho_F^{\Omega_r}(x) \le \rho_F^{\Omega_r}(x) + r,
\]
which therefore justifies the inclusion ``$\subset$" in \eqref{r-enlar}.

To verify the reverse inclusion in  \eqref{r-enlar}, suppose that \(\rho_F^\Omega(x) = t\) meaning that \(x \in \Omega + tF\). This gives us \(r < t\) due to \(x \notin \Omega_r\). Using the convexity of $F$ ensures that $tF=(r+t-r)F=rF+(t-r)F$, and we have $x \in \Omega + rF + (t - r)F$. It follows from Lemma~\ref{lemma:r_enlargement_other_shape} by the convexity of \(F\)  that 
$x \in \Omega_r + (t - r)F$ yielding the relationships
$$
\rho_F^{\Omega_r}(x) \leq t - r = \rho_F^\Omega(x) - r,
$$
which readily justify the inclusion ``$\supset$" in \eqref{r-enlar} and thus completes the proof.	
\end{proof}

\begin{lemma}\label{lemma:minkowski_estimate}  Under the assumptions of Lemma~{\rm\ref{lemma:equality_mikowski_to_omega_r}}, let $t\ge 0$, $f \in F$, and $x \in \textnormal{dom}\  \rho_F^\Omega$. Then 
\begin{equation}\label{rhoF}
\rho_F^\Omega(x+tf)\le \rho_F^\Omega(x)+t.
\end{equation}
\end{lemma}
\begin{proof}
Given $\epsilon > 0$, choose $s \ge 0$ such that $\rho_F^\Omega(x) \le s < \rho_F^\Omega(x) + \epsilon$ giving us $x \in \Omega + hF$. Since $tf \in tF$, it follows that
\[
x + tf \in \Omega + sF + tF = \Omega + (s + t)F,
\]
where the equality holds due to the convexity of $F$. Therefore,
\(\rho_F^\Omega(x + tf) \le s + t < \rho_F^\Omega(x) + \epsilon + t.\)
Taking the limit as $\epsilon\dn 0$, we arrive at \eqref{rhoF} and thus complete the proof.
\end{proof}

Yet another important property  of the set-based Minkowski gauge is revealed below.\vspace*{0.05in}

\begin{proposition} \label{prop:f_on_border} Let $F\subset \mathbb{R}^q$ be a convex compact in $\mathbb{R}^q$ with $0 \in \textnormal{int}\,F$, and let $\Omega \subset\mathbb{R}^q$ be a nonempty, closed, and convex set. If $\rho_F^{\Omega}(\bar{x})=t>0$, then we have the decomposition $\bar{x}=\bar{\omega}+tf$ with some 
$\bar{\omega}\in\textnormal{bd}\,\Omega$ and $f \in \textnormal{bd}\,F$.
\end{proposition}
\begin{proof}
It follows from definition (\ref{eq:generalized_projection}) of the generalized projection that 
\[
\rho_F^{\Omega}(\bar{x}) = \rho_F(\bar{x} - \bar{\omega}) = t\;\mbox {for any }\; \bar{x} \notin \Omega\;\mbox{ and }\;\bar{\omega} \in \Pi_F(\bar{x}; \Omega),
\]
Employing Proposition~\ref{prop:projection&border} tells us that \(\bar{\omega} \in \textnormal{bd}\,\Omega\). Therefore, \(\bar{x} - \bar{\omega} \in tF\), i.e., \(\bar{x} \in \bar{\omega} + tF\). Consequently, there exists \(f \in F\) such that \(\bar{x} = \bar{\omega} + tf\). 

Suppose now on the contrary that \(f \in \textnormal{int}\,F\). Then there exists a number \(\epsilon > 0\) such that \(\mathbb{B}(f; \epsilon) \subset F\). Consider the vector
\[
v: = \frac{\bar{x} - \bar{\omega}}{\|\bar{x} - \bar{\omega}\|}
\]
and deduce from \(f + \epsilon v \in F\) that \(t(f + \epsilon v) \in tF\). This leads us to the implication
\[
\big[tf + t\epsilon v \in tF\big]\Longrightarrow\Big[\bar{x} - \bar{\omega} + \frac{t\epsilon(\bar{x} - \bar{\omega})}{\|\bar{x} - \bar{\omega}\|} = (\bar{x} - \bar{\omega})\Big(1 + \frac{t\epsilon}{\|\bar{x} - \bar{\omega}\|}\Big) \in tF\Big].
\]
Therefore, we get the inclusion 
\[
\bar{x} - \bar{\omega} \in \frac{t}{1 + \frac{t\epsilon}{\|\bar{x} - \bar{\omega}\|}}F.
\]
Since \(t > 0\) and \(\epsilon > 0\), it follows that
\[
t^{\prime}: = \frac{t}{1 + \frac{t\epsilon}{\|\bar{x} - \bar{\omega}\|}} < t.
\] 
This is a contradiction, since we found a number \(t^{\prime}>0\) such that \(\bar{x} - \bar{\omega} \in t^{\prime}F\) and \(t^{\prime} < \rho_F^{\Omega}(\bar{x})\). Thus we arrive at \(f \in \textnormal{bd}\,F\) as claimed.
\end{proof}

The following proposition establishes a relationship for the minimal time to reach from a given point $\ox\in \mathbb{R}^q$ to the {\em closest point} in the reference set $\Omega \subset \mathbb{R}^q$, \textcolor{black}{see Fig.~\ref{fig:mdf-msmg}(c) for an illustrative example}.\vspace*{0.05in}

\begin{proposition}\label{prop:SMG_with_negative_F} Let $F \subset \mathbb{R}^q$ be a closed and convex set with $0 \in \textnormal{int}\,F$, and let $\Omega$ be an arbitrary nonempty subset of $\mathbb{R}^q$. Then we have the relationship
\begin{equation}\label{min-point}
\inf_{\omega\in \Omega} \rho_F^{\{\bar{x}\}}(\omega)= \rho_{-F}^{\Omega}(\bar{x}).
\end{equation}
\end{proposition}
\begin{proof}
Fix any $\epsilon>0$ and choose $t >0$ such that $\rho_{-F}^{\Omega}(\bar{x}) \le t < \rho_{-F}^{\Omega}(\bar{x})+\epsilon$. We can find $\omega \in \Omega$ and $f \in F$ with $\bar{x}=\omega+t(-f)$. Therefore, $\omega=\bar{x}+tf$, and it follows that
\[
\inf_{\omega\in \Omega} \rho_F^{\{\bar{x}\}}(\omega) \le  \rho_F^{\{\bar{x}\}}(\omega)\le t < \rho_{-F}^{\Omega}(\bar{x})+\epsilon.
\]
Passing to the limit as $\ve\dn 0$	leads us to the inequality ``$\le$" in \eqref{min-point}. To verify the reverse inequality in \eqref{min-point}, take $\alpha>0$, and choose a number $s$ with $\inf_{\omega\in \Omega} \rho_F^{\{\bar{x}\}}(\omega) \le s< \inf_{\omega\in \Omega} \rho_F^{\{\bar{x}\}}(\omega)+\alpha$. This gives us $\omega \in \Omega$ and $f\in F$ such that $\omega=\bar{x}+sf$. Therefore, $\bar{x}=\omega+s(-f)$, and we have
\[
\rho_{-F}^{\Omega}(\bar{x})\le <s<
\inf_{\omega\in \Omega} \rho_F^{\{\bar{x}\}}(\omega)+\alpha.
\]
Passing to the limit  $\alpha\dn 0$ verifies the inequality ``$\ge$" in \eqref{min-point} and thus completes the proof.
\end{proof}

\begin{figure}
    \centering
    \includegraphics[width=1\linewidth]{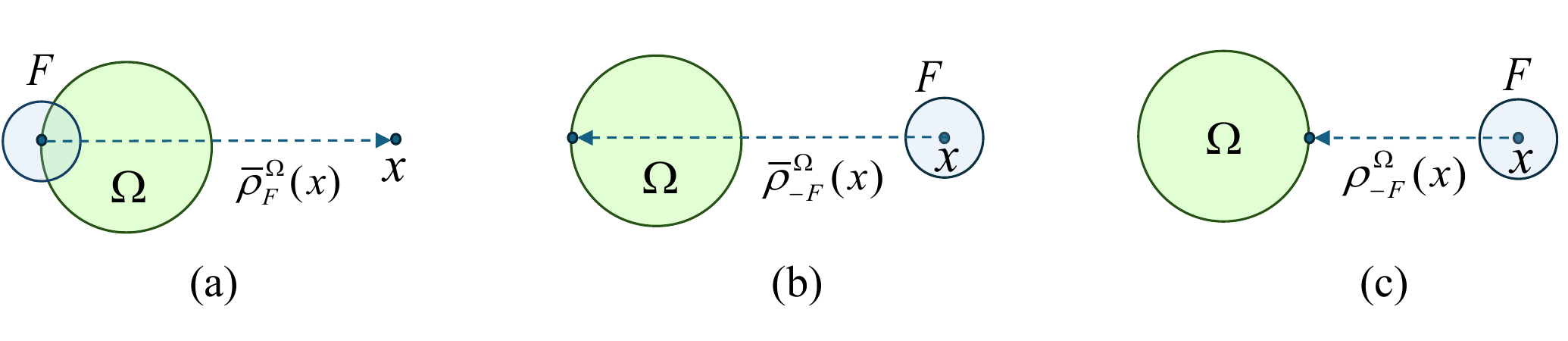}
    \caption{\textcolor{black}{This figure presents three simple illustrative examples of different gauge functions in the case where both the reference and dynamic sets are circles. (a) provides an example of the MSMG function \eqref{eq:MSMG_function}. (b) illustrates the MSMG function with $-F$, as investigated in Proposition~\ref{prop:MSMG_with_negative_F}. (c) shows an example of the set-based Minkowski gauge function \eqref{eq:set-based-minkowski-definition} with $-F$, as studied in Proposition~\ref{prop:SMG_with_negative_F}.}
}
    \label{fig:mdf-msmg}
\end{figure}

Next we derive an equality allowing us to calculate the minimal time of moving from a point $\bar{x}\in \mathbb{R}^q$ to the {\em farthest point} in the reference set $\Omega \subset \mathbb{R}^q$, \textcolor{black}{see Fig.~\ref{fig:mdf-msmg}(b) for an example.}\vspace*{0.05in}

\begin{proposition}\label{prop:MSMG_with_negative_F} Suppose in addition to the assumptions of Proposition~{\rm\ref{prop:SMG_with_negative_F}} that the set $\Omega$ is bounded. Then we have the equality
\begin{equation}\label{max-point}
\sup_{\omega \in \Omega}\rho_F^{\{\bar{x}\}}(\omega)=\Bar{\rho}_{-F}^{\Omega}(\bar{x})
\end{equation}
\end{proposition}
\begin{proof} Fix \(\epsilon > 0\) and choose \(t \ge 0\) such that  
\[
\Bar{\rho}_{-F}^{\Omega}(\bar{x}) \le t < \Bar{\rho}_{-F}^{\Omega}(\bar{x}) + \epsilon.
\]  
This tells us that \(\bar{x} - \Omega \subset tF\), or equivalently, we have \(\bar{x} - \omega \in t(-F)\) for all \(\omega \in \Omega\), i.e., \(\omega - \bar{x} \in tF\). Thus $\rho_F(\omega - \bar{x}) \le t$ whenever $\omega \in \Omega$. It follows from Proposition~\ref{prop:relationship_between_set_based_and_Minkowski} that \(\rho_F^{\{\bar{x}\}}(\omega) = \rho_F(\omega - \bar{x})\). This leads us to the relationships  
\[
\sup_{\omega \in \Omega} \rho_F(\omega - \bar{x}) = \sup_{\omega \in \Omega} \rho_F^{\{\bar{x}\}}(\omega) \le t < \Bar{\rho}_{-F}^{\Omega}(\bar{x}) + \epsilon.
\]  
Letting \(\epsilon\dn0\), we arrive at the inequality ``$\le$ in \eqref{max-point}.  
	
To verify the  reverse inequality,  denote 
$s:=\sup_{\omega \in \Omega} \rho_F(\omega - \bar{x}) = \sup_{\omega \in \Omega} \rho_F^{\{\bar{x}\}}(\omega)$. Then we have \(\rho_F(\omega - \bar{x}) \le s\), which implies in turn that \(\omega - \bar{x} \in sF\). Hence \(\Omega - \bar{x} \subset sF\), or equivalently, \(\bar{x} - \Omega \subset s(-F)\). This justifies the inequality ``$\ge$" in \eqref{max-point} and thus completes the proof.  
\end{proof}

The following result of its own interest provides new formulas, in comparison with Theorem~\ref{prop:subdifferential2}, to calculate the subdifferential of the set-based Minkowski gauge function at out-of-set points by using generalized projections.\vspace*{0.05in}

\begin{theorem}\label{sub-proj} Let $F$ be a convex and compact set in $\mathbb{R}^q$ with $0 \in \textnormal{int}\,F$, and let $\Omega\subset\mathbb{R}^q$ be nonempty, closed, and convex. Given 
$\bar{x} \notin \Omega$ and $\bar{\omega}\in \Pi_F(\bar{x};\Omega)$, denote $r:=\rho_F^{\Omega}(\bar{x})$. Then we have the subdifferential representations
\begin{equation}\label{sub-proj1}
\partial\rho_F^{\Omega}(\bar{x})=\partial \rho_F(\bar{x}-\bar{\omega}) \cap N(\bar{x};\Omega_r),
\end{equation}
\begin{equation}\label{eq:subdifferential_exact_approximation}
\partial \rho_F^\Omega(\bar{x})=\partial \rho_F(\bar{x}-\bar{\omega}) \cap N(\bar{\omega};\Omega).
\end{equation}
\end{theorem}
\begin{proof}
Proposition~\ref{prop:generalized_projection_nonempty} tells us that $\Pi_F(\bar{x}; \Omega)\ne\emp$. It follows from Proposition~\ref{prop:subdifferential_estimation} that $\partial \rho_F^\Omega(\bar{x}) \subset \partial \rho_F(\bar{x} - \bar{\omega})$ and from 
Proposition~\ref{prop:subdifferential2} that $\partial \rho_F^\Omega(\bar{x}) \subset N(\bar{x}; \Omega_r)$. Combining the above, we arrive at the inclusion ``$\subset$" in \eqref{sub-proj1}.

To verify the reverse inclusion in \eqref{sub-proj1}, it suffices to show that $\partial \rho_F(\bar{x}-\bar{\omega}) \subset \Upsilon$, where $\Upsilon$ is defined in 
Proposition~\ref{prop:subdifferential2}. Note that $\rho_F^{\{0\}}(x) = \rho_F(x)$, which implies  that $\partial \rho_F^{\{0\}}(x) = \partial \rho_F(x)$ for all $x \in \textnormal{dom}\,\rho_F$.  
Taking there the point $\bar{x} - \bar{\omega}$ with $\rho_F(\bar{x}-\bar{\omega})=\rho_F^{\{0\}}(\bar{x} - \bar{\omega}) = r$, we pick $v \in \partial \rho_F(\bar{x} - \bar{\omega})$ and conclude that $v \in \partial \rho_F^{\{0\}}(\bar{x} - \bar{\omega})$. Moreover, 
Proposition~\ref{prop:subdifferential2} tells us that
\[
v \in N\big(\bar{x} - \bar{\omega};\{0\} + rF\big) \cap \Upsilon
\]
yielding $v \in \Upsilon$ and $\partial \rho_F(\bar{x}-\bar{\omega}) \cap N(\bar{x};\Omega_r) \subset N(\bar{x};\Omega_r) \cap \Upsilon$. It follows from 
Proposition~\ref{prop:subdifferential2} that
$$
\partial \rho_F^{\Omega}(\bar{x})=N(\bar{x};\Omega_r) \cap \Upsilon,
$$
which justifies the fulfillment of the inclusion ``$\supset$" in \eqref{sub-proj1} and the equality therein.

To verify (\ref{eq:subdifferential_exact_approximation}), let us first show that $\partial \rho_F^\Omega(\bar{x}) \subset N(\bar{\omega};\Omega)$. Pick $v \in \partial \rho_F^{\Omega}(\bar{x})$ and fix $\bar{\omega} \in \Pi_F(\bar{x}; \Omega)$, which gives us $\bar{x} - \bar{\omega} \in rF$ and then $\bar{f}\in F$ with $\bar{x}-\bar{\omega}=r\bar{f}$. Select now any $y \in G(\bar{x})$, where the set $G(\bar{x})$ is defined by
\[
G(\bar{x}) := \big\{ \omega + (\bar{x} - \bar{\omega})\;\big|\;\omega \in \Omega,\;\bar{\omega} \in \Pi_F(\bar{x};\Omega)\big\}.
\]
Thus we get the representation $y = \omega + r\bar{f}$ with some $\omega \in \Omega$. Recalling that 
$\rho_F^{\Omega}(y) \le r$ and using the subgradient definition for $v \in \partial \rho_F^{\Omega}(\bar{x})$ lead us to
\[
\langle v, y - \bar{x} \rangle = \langle v, \omega - \bar{\omega} \rangle \le \rho_F^{\Omega}(y) - \rho_F^{\Omega}(\bar{x}) \le 0.
\]
which shows that  $v \in N(\bar{\omega}; \Omega)$ and hence justifies that $\partial\rho_F^{\Omega}(\bar{x})\subset N(\bar{\omega};\Omega)$. Taking into account that $\partial\rho_F^{\Omega}(\bar{x}) \subset \partial \rho_F(\bar{x}-\bar{\omega})$ by Proposition~\ref{prop:subdifferential_estimation}, we arrive at the inclusion ``$\subset$" in \eqref{eq:subdifferential_exact_approximation}.

It remains to verify the reverse inclusion ``$\supset$" in \eqref{eq:subdifferential_exact_approximation}, which would follow from 
\begin{equation}\label{sub-proj2}
\partial \rho_F(\bar{x} - \bar{\omega}) \cap N(\bar{\omega}; \Omega) \subset \partial \rho_F(\bar{x} - \bar{\omega}) \cap N(\bar{x}; \Omega_r) = \partial \rho_F^{\Omega}(\bar{x}).
\end{equation}
To get \eqref{sub-proj2}, it suffices to show that \(v \in \partial \rho_F(\bar{x} - \bar{\omega}) \cap N(\bar{\omega}; \Omega)\) yields \(v \in N(\bar{x}; \Omega_r)\). To proceed with the latter, pick any \(x \in \Omega_r\) and recall that  $\Pi_F(x;\Omega)\neq \emptyset$ by Proposition~\ref{prop:generalized_projection_nonempty}. Then find $f\in F$ and take $x=\omega+tf$ with $w\in \Pi_F(x;\Omega)$ and $t\le r$. Since \(v \in \partial \rho_F(\bar{x} - \bar{\omega}) = \partial \rho_F^{\{0\}}(\bar{x})\), we get \(v \in \Upsilon\) and \(\langle v, f \rangle \le \sigma_F(v) \le 1\). Therefore,
\[
\begin{split}
\langle v, x - \bar{x} \rangle &= \langle v, \omega + tf - \bar{x} \rangle \\
&= t \langle v, f \rangle + \langle v, \omega - \bar{\omega} \rangle + \langle v, \bar{\omega} - \bar{x} \rangle \\
&\le t + \langle v, \omega - \bar{\omega} \rangle + \langle v, \bar{\omega} - \bar{x} \rangle \\
&\le r + \langle v, \omega - \bar{\omega} \rangle + \langle v, \bar{\omega} - \bar{x} \rangle.
\end{split}
\]
Since \(v \in N(\bar{\omega}; \Omega)\), it follows that \(\langle v, \omega - \bar{\omega} \rangle \le 0\), and \(v \in \partial \rho_F(\bar{x} - \bar{\omega})\) implies that
\[
\langle v, \bar{\omega} - \bar{x} \rangle = \langle v, -(\bar{x} - \bar{\omega}) \rangle \le \rho_F(0) - \rho_F(\bar{x} - \bar{\omega}) = -\rho_F^{\Omega}(\bar{x})=-r.
\]
Thus for all \(x \in \Omega_r\), we get \(\langle v, x - \bar{x} \rangle \le 0\) meaning that \(v \in N(\bar{x}; \Omega_r)\). This brings us to \eqref{sub-proj2} and therefore completes the proof of the theorem.
\end{proof}

The following observation provides a property of \( \rho_F \) used in the proof of Theorem~\ref{theorem:MSMG_properties}.\vspace*{0.05in}

\begin{lemma}\label{lemma:Minkowski_of_F_and_negative_F}
Let \( F \subset \mathbb{R}^q \) be a closed and convex set with \( 0 \in \operatorname{int}\,F \). Then the Minkowski gauge function enjoys the property $\rho_F(x) = \rho_{-F}(-x)$ for all \( x \in \mathbb{R}^q \).
\end{lemma}
\begin{proof}
Denoting  \( t := \rho_F(x) \), we get from the definition of the Minkowski gauge~\eqref{eq:minkowski_definition} that \( x \in tF \). This yields \( -x \in t(-F) \), which ensures that \( \rho_{-F}(-x) \le t = \rho_F(x) \). Conversely, letting \(s := \rho_{-F}(-x) \)  gives us \( x \in dF \), which tells us that \( \rho_F(x) \le s = \rho_{-F}(-x) \). Combining both inequalities verifies the claimed result. 
\end{proof}

Now we are ready to establish a precise  relationship between the set-based Minkowski gauge function \( \rho_F^{\Omega} \) and the \emph{minimal time function} \( T^F_{\Omega} \), defined in \parencite{mordukhovich2023easy} by
\begin{equation}\label{eq:minimal_time_function}
T^F_{\Omega}(x) := \inf\big\{ t \ge 0\;\big|\;(x + tF) \cap \Omega \neq \emptyset\big\}, \quad x \in \mathbb{R}^q.
\end{equation}

\begin{lemma}\label{lemma:relationship_SMG_MinimalTime} Let $F\subset \mathbb{R}^q$ be a closed and convex set with $0 \in \textnormal{int}\,F$, and let $\Omega \subset \mathbb{R}^q$ be an arbitrary nonempty set. Then we have the equality $\rho_{-F}^{\Omega}(x)=T^F_{\Omega}(x)$ for $x \in \mathbb{R}^q$.
\end{lemma}
\begin{proof}
It follows from Propositions~\ref{prop:relationship_between_set_based_and_Minkowski} and \ref{prop:SMG_with_negative_F} that
\[\rho_{-F}^{\Omega}(x)=\inf_{\omega \in \Omega} \rho_F^{\{x\}}(\omega)=\inf_{\omega \in \Omega} \rho_F(\omega-x).\]
By Theorem 6.19 in \parencite{mordukhovich2023easy} tells us  that $T_{\Omega}^{F}(x)=\inf_{\omega \in \Omega} \rho_F (\omega-x)$, which justifiers the claimed relationship.
\end{proof}

Finally in this section, we establish a precise relationship between the MSMG function \eqref{eq:MSMG_function} of our interest and the \emph{maximal time function} $C_{\Omega}^F$, which has been studied for some versions of the generalized Fermat--Torricelli and generalized Sylvester models in \parencite{nam2013generalized,nam2013nonsmooth_maximal_time_function} being defined by\vspace*{0.03in}
\begin{equation}\label{eq:Maximal_time_function_definition}
C_{\Omega}^F(x) := \inf\big\{t \ge 0 \;\big|\; \Omega \subset x + tF\big\}, \quad x \in \mathbb{R}^q.
\end{equation} 

\begin{lemma}\label{lemma:relationship_MSMG_Maximal_Time_Function} Let $F \subset \mathbb{R}^q$ be a closed, bounded, and convex set with $0 \in \operatorname{int}\,F$, and let $\Omega \subset \mathbb{R}^q$ be a nonempty bounded set. Then we have $\bar{\rho}_{-F}^{\Omega}(x) = C_{\Omega}^F(x)$ for all $x \in \mathbb{R}^q$.
\end{lemma}
\begin{proof}
It follows from Proposition~\ref{prop:relationship_MSMG_Minkowski} that $\bar{\rho}_{-F}^{\Omega}(x)=\sup_{\omega \in \Omega} \rho_{-F}(x-\omega)$. We aim to show that 
\begin{equation}\label{max-time}
\sup_{\omega \in \Omega} \rho_{-F}(x-\omega)=\sup_{\omega \in  \Omega} \rho_{F}(\omega-x).
\end{equation}
To proceed, take $\epsilon>0$ and $t:=\sup_{\omega \in \Omega} \rho_{-F}(x-\omega)-\epsilon$ with $t>0$. There exists $\omega \in \Omega$ such that $x-\omega \notin t(-F)$, which gives us $\omega-x \notin tF$. Therefore, $t<\rho_{F}(\omega-x)\le \sup_{\omega \in \Omega} \rho_{F}(\omega-x)$ and
\begin{equation*}
\big[\sup_{\omega \in \Omega} \rho_{-F}(x-\omega)-\epsilon< \sup_{\omega \in \Omega} \rho_{F}(\omega-x)\big]\Longrightarrow\big[\sup_{\omega \in \Omega} \rho_{-F}(x-\omega) \le \sup_{\omega \in \Omega} \rho_{F}(\omega-x)\big]
\end{equation*}
as $\ve\dn 0$ while justifying the inequality ``$\le$ in \eqref{max-time}. Similarly to the above,  we can derive the reverse inequality  in \eqref{max-time}. Having this in hand and using \parencite[Proposition~1]{mordukhovich2013smallest} give us $C_{\Omega}^{F}(x)=\sup_{\omega \in \Omega} \rho_F(\omega-x)$ and leads us therefore to
\begin{equation*}
\bar{\rho}_{-F}^{\Omega}(x)=\sup_{\omega \in \Omega} \rho_{-F}(x-\omega) = \sup_{\omega \in \Omega} \rho_{F}(\omega-x)= C_{\Omega}^{F}(x),
\end{equation*}
which verifies \eqref{max-time} and thus completes the proof.
\end{proof}

\printbibliography

\end{document}